\documentclass[12pt,draftcls,onecolumn]{IEEEtran}

\usepackage{graphics} % for pdf, bitmapped graphics files
\usepackage{epsfig} % for postscript graphics files
\usepackage{amsmath} % assumes amsmath package installed
\usepackage{amssymb}  % assumes amsmath package installed
\usepackage{cite}
\usepackage{bm}
\usepackage{subfigure}
\usepackage{acronym}
\usepackage{paralist}
\usepackage{float}
\usepackage{epstopdf}

\newtheorem{critical point}{Theorem}
\newtheorem{integral curves}[critical point]{Theorem}
\newtheorem{reflection}[critical point]{Theorem}
\newtheorem{Nobstacles}[critical point]{Theorem}
\newtheorem{Multi-Robot}[critical point]{Theorem}
\newtheorem{vectorfield}{Definition}
\newtheorem{integralcurve}[vectorfield]{Definition}
\newtheorem{singularpoint}[vectorfield]{Definition}
\newtheorem{1obstacle}{Lemma}

\newtheorem{nocollisions}[1obstacle]{Lemma}
\newtheorem{Remark}{Remark}

\DeclareMathOperator{\atan2}{atan2}

\DeclareMathOperator{\F}{\mathrm F}

\DeclareMathOperator{\R}{\mathbb R}

\DeclareMathOperator{\N}{\textrm N}

\begin{document}

\title{Motion Planning and Collision Avoidance using Non-Gradient Vector Fields}
\author{Dimitra Panagou
\thanks{Dimitra Panagou is with the Department of Aerospace Engineering, University of Michigan, Ann Arbor, MI, USA; \texttt{dpanagou@umich.edu}.}}

\markboth{IEEE Transactions on }%
{Panagou:Multi}

\maketitle

\begin{abstract}
This paper presents a novel feedback method on the motion planning for unicycle robots in environments with static obstacles, along with an extension to the distributed planning and coordination in multi-robot systems. The method employs a family of 2-dimensional analytic vector fields, whose integral curves exhibit various patterns depending on the value of a parameter $\lambda$. More specifically, for an a priori known value of $\lambda$, the vector field has a unique singular point of dipole type and can be used to steer the unicycle to a goal configuration. Furthermore, for the unique value of $\lambda$ that the vector field has a continuum of singular points, the integral curves are used to define flows around obstacles. An almost global feedback motion plan can then be constructed by suitably blending attractive and repulsive vector fields in a static obstacle environment. The method %does not thus require the definition of Lyapunov-like or potential functions encoding the avoidance objective, while it 
does not suffer from the appearance of sinks (stable nodes) away from goal point. Compared to other similar methods which are free of local minima, the proposed approach does not require any parameter tuning to render the desired convergence properties. %; the values of the parameter $\lambda$ involved in the definition of the vector field are a priori known. 
The paper also addresses the extension of the method to the distributed coordination and control of multiple robots, where each robot needs to navigate to a goal configuration while avoiding collisions with the remaining robots, and while using local information only. More specifically, based on the results which apply to the single-robot case, a motion coordination protocol is presented which guarantees the safety of the multi-robot system and the almost global convergence of the robots to their goal configurations. The efficacy of the proposed methodology is demonstrated via simulation results in static and dynamic environments. 
\end{abstract}

\acrodef{wrt}[w.r.t.]{with respect to}
\acrodef{apf}[APF]{Artificial Potential Fields}
\acrodef{ges}[GES]{Globally Exponentially Stable}

\section{Introduction}

Motion planning, coordination and control for robotic systems still remains an active research topic in many respects. The primary motivation has been the computation of safe, collision-free trajectories for robotic agents, mechanisms and autonomous vehicles which operate in constrained and/or uncertain environments. Research within the robotics community has attributed various formulations and methodologies on the motion planning problem, often specialized based on the control objectives and the characteristics of the problems at hand. These methodologies range from Lyapunov-based control methods, to sampling-based planning, to combinatorial planning, to formal methods \cite{Parker_2009, Kress-Gazit_TRO09, Bhatia_RAM11}. 
Multi-robot systems have attracted the interest of the control systems community as well. Emphasis has been given in consensus, flocking and formation control problems for multiple agents \cite{Ren_Cao_2011}.

Avoiding obstacles and inter-agent collisions is a requirement of highest priority in motion planning and coordination problems. Recently, significant interest has been paid to the high-level task planning under complex goals, where the problem for an autonomous robot has transitioned from the classical motion planning formulation (i.e., move from point $A$ to point $B$) to the consideration of complex goals under temporal specifications; such specifications are typically described as: ``visit region $A$, and then visit either region $B$ or region $C$". Despite the tremendous and elegant contributions in this area, which provide elegant solutions to the high-level mission synthesis with rigorous guarantees under certain assumptions on the considered environments \cite{Kress-Gazit_TRO09, Bhatia_RAM11}, the interconnection of high-level tasking with the physical layer/system is still an open problem in many respects. One issue is the consideration of multiple agents in dynamic environments and the associated complexity in finding provably correct solutions in the presence of nonlinearities, arbitrary constraints, and uncertainty.

The scope of this paper is to provide a solution to the motion planning problem for single and multiple nonholonomic agents in dynamic environments, where agents have local sensing and communication capabilities and which may be populated by dynamic (moving) obstacles. Our goal is to provide a feedback synthesis of low-level planning controllers along with certain guarantees, which can later on be combined with high-level tasks, such as dynamic coverage \cite{Panagou_CDC14}, towards provably correct feedback solutions for a specific class of dynamical systems in dynamic environments. The technical tools which we use towards this goal are set-invariance methods, which have been proved efficient in constrained control problems of a class of nonlinear, under-actuated systems \cite{Panagou_Automatica2013}. 

The spirit of the proposed solutions is similar, \emph{but not identical to}, Lyapunov-like scalar functions, such as the Avoidance Functions in \cite{Leitmann} and the \ac{apf} in \cite{Khatib86, HernandezMartinez_2011}. More specifically: It is well-known that, although scalar functions offer the merit of Lyapunov-based control design and analysis, yielding thus solutions in closed-form with certain guarantees \cite{Stipanovic_MonotoneApproximations2012}, they suffer from the drawback of possible local minima away from the goal point, i.e., of points in the state space other than the desired equilibrium at which the gradient vector vanishes; this in principle results in system trajectories which get stuck away from the goal point. Certain forms of potential functions may overcome this limitation; namely, navigation functions \cite{Rimon_Koditscek_1992} and harmonic functions \cite{Connolly, Kozlowski}, but under some cost: the caveat in the former case is that the Morse property which guarantees the non-existence of local minima is rendered after a tuning parameter exceeds a lower bound, which is not a priori known. In the latter case, harmonic functions may be constructed with either discrete or continuous approaches, but the computational cost of discrete methods is quite demanding. Continuous approaches which employ the analogies of Laplace equation with fluid mechanics yield closed-form solutions for certain dynamic environments \cite{Feder_ICRA1997}. Stream functions \cite{Waydo_ICRA2003} combine the local-minima-free property of harmonic functions along with hydrodynamic concepts to yield streamlines which may be preferable for second order systems. The method of vortex fields \cite{DeLuca_Oriolo_ICRA1994} uses the anti-gradient of a scalar function to define flows around obstacles. 

Now, let us note that one common ground in this class of solutions is the resulting \emph{gradient vector field} which is employed in the control synthesis. In this respect, the idea of \emph{directly} defining vector fields encoding obstacle avoidance has been studied for robot motion planning problems. In \cite{Lindemann_Lavalle_IJRR09}, for instance, simple smooth vector fields are locally constructed in given convex cell decompositions of polygonal environments, so that their integral curves are by construction collision-free and, in a sequential composition spirit, convergent to a goal point. The method, nevertheless, presumes the existence of a high-level discrete motion plan which determines the successive order of the cells from an initial to a final configuration.
Recent work employing vector fields for vehicles' navigation is presented also in \cite{Liddy_2008} %which is based on the virtual force field method \cite{Borenstein_1989}, 
and in \cite{Martinez_IECON09}. The approach with velocity vector fields in \cite{Dixon_IWRMC2005} is also relevant to the context. However, these contributions address only the position control of the robot, while the orientation is not guaranteed to converge to a desired value. %In contrast, our approach guarantees the convergence of the orientation trajectories to any predefined value, thanks to the dipolar nature of the considered vector fields.

Stepping now a little further away from single-agent problems: when it comes to multiple agents, their motion towards goal configurations defines a dynamic environment and poses challenges to the planning, coordination and control design, even in the absence of static physical obstacles. At the same time, limitations in the available sensing and communication platforms impose additional constraints to the multi-agent system. Given a pair $(i,j)$ of agents $i$ and $j$, agents typically make decisions on their actions based on available information, which can be either locally measured using onboard sensors, or transmitted and received across the nodes of the multi-agent system via wireless communication links. Thus, information flow between two agents can be either bidirectional (undirected) or unidirectional (directed). During the past ten years, research efforts have achieved the formalization of problems such as consensus and formation control in multi-agent networks using tools and notions from graph theory, matrix theory and Lyapunov stability theory \cite{GraphTheoretic_Egerstedt, Ren_CSM07, Olfati-Saber_IEEE07, Dimarogonas08_TRO, Loizou08_TRO}. The case of directed information exchange has recently attracted increased interest \cite{LinLiSCL08, Mei_Automatica_2012, Qu_Li_IJRNC12, Yuetal2014, Ren_TAC14}, motivated in part by the fact that undirected information flow is not always a realistic and practical assumption, due to bandwidth limitations in the network, anisotropic sensing of the agents etc. Extending consensus algorithms to nonlinear systems has also become popular, see for instance \cite{Ren_CDC07, LiuWen_SCL13}. 

Nevertheless, despite that consensus, flocking, and formation control algorithms achieve collision avoidance in multi-vehicle systems by carefully selecting initial conditions and controlling relative distance and heading, they are typically not used in encoding problems such as navigation to \emph{specific} goal locations for each one of the agents. In this respect, the development of planning and coordination algorithms for the motion of multiple agents along with safety and performance guarantees is an open problem in many respects. 

\subsection{Overview}
This paper presents a novel method on the motion planning and coordination in environments with static and/or dynamic obstacles, which results in feedback motion plans for unicycle robots along with collision avoidance guarantees. The method employs a family of two-dimensional analytic vector fields, originally introduced in \cite{Panagou_CDC11A}, given as:
\begin{align}
\label{2-d dipolar}
\mathbf F(\bm r) = \lambda (\bm p^T \bm r)\bm r - \bm p(\bm r^T \bm r),
\end{align}
where $\lambda\in\mathbb R$ is a parameter to be specified later on, $\bm r=\left[x\;\;y\right]^T$ the position vector \ac{wrt} a global cartesian frame and $\bm p=\left[p_x\;\;p_y\right]^T$, with $\bm p\neq \bm 0$.\footnote{The role the vector $\bm p\in \R^2$ plays in the properties of the vector field \eqref{2-d dipolar} becomes evident later on in Theorem \ref{reflection}.}

In \cite{Panagou_CDC11A} the family of vector fields \eqref{2-d dipolar} was employed in the control design for steering kinematic, drift-free systems in chained form in \emph{obstacle-free} environments. %In those designs the parameter $\lambda$ was typically set equal to $\lambda=3$. The justification of this becomes evident in the sequel.

In this paper we first show that, except for a \emph{known} value of the parameter $\lambda$, the vector field \eqref{2-d dipolar} has a unique singular point on $\R^2$. More specifically:
\begin{inparaenum}
\item[(\emph{i})] For $\lambda>1$ the pattern of the integral curves around the unique singular point is dipolar \cite{Henle}. Such vector field can be used for steering a unicycle to a goal configuration.
\item[(\emph{ii})] For $\lambda=1$ the vector field has a continuum of singular points and can be used to define tangential flows around circular obstacles.
\item[(\emph{iii})] For $\lambda<0$ the pattern of the integral curves is suitable for defining repulsive flows away from lines, and as thus, away from polygonal obstacles. A preliminary example is given in the Appendix of \cite{Panagou_Technical_Report_Vector_Fields}.
\end{inparaenum}

We then consider the single-agent case in a static environment of circular obstacles and propose a blending mechanism between attractive and repulsive vector fields, which yields almost global feedback motion plans. In other words, we construct vector fields whose integral curves are convergent to a goal configuration, except for a set of initial conditions of Lebesgue measure zero, and collision-free by construction. This in turn results in simple feedback control laws, which force the system to flow along the vector field. %The integral curves in the vicinity of undesired singularities are shown to be similar to those in the vicinity of saddle points, i.e., of unstable equilibria.

We finally consider the extension of the methodology to the distributed coordination and control for multiple nonholonomic agents. Based on the results for the single-agent case in static obstacle environments, we propose a coordination protocol for multiple agents which need to converge to specific goal configurations, using local information only. The proposed protocol yields collision-free and almost globally convergent trajectories for the multi-agent system.   

\subsection{Contributions and Organization}
When it comes to the single-agent case, i.e., to a robot operating in a known, static environment of circular obstacles, the proposed method does not suffer from the appearance of sinks (stable nodes) away from goal point. Furthermore, compared to similar feedback methods which rely on scalar (potential) functions, such as \cite{Rimon_Koditscek_1992}, the main difference and advantage of the proposed approach is that:
\begin{enumerate}
\item[(i)] no parameter tuning is needed in order to render the desired convergence properties; the values of the parameter $\lambda$ of the vector field are known \emph{a priori}.
\end{enumerate}
Compared to similar methods which rely on vector fields, such as \cite{Lindemann_Lavalle_IJRR09}, the proposed method:
\begin{enumerate}
\item[(ii)] requires neither the computation of a cell decomposition of the free space, nor the existence of a high-level discrete motion plan, and as thus it is free of any computational complexity issues,
\item[(iii)] addresses the motion planning and collision avoidance for multiple agents in dynamic environments, and is scalable as the number of agents increases.
\end{enumerate}
Finally, compared to other similar vector field based methods, such as \cite{Dixon_IWRMC2005, Liddy_2008, Martinez_IECON09}, the proposed method:
\begin{enumerate}
\item[(iv)] guarantees the convergence of the orientation trajectories of the robots to any predefined value.
\end{enumerate}
\begin{Remark}
While here we consider circular, not polygonal, obstacle environments, preliminary results reveal that the method can be used for defining repulsions around polygonal obstacles as well, see the Appendix in \cite{Panagou_Technical_Report_Vector_Fields}.
\end{Remark}

When it comes to the multi-agent case, i.e., to multiple agents moving towards goal configurations while avoiding collisions, the proposed method:
\begin{enumerate}
\item[(v)] offers the flexibility to directly impose the minimum allowable clearance among agents, something which typically is not the case with gradient-based solutions. This characteristic might be desirable, for instance, when considering multi-robot systems in confined environments. 
\item[(vi)] being a non-gradient vector field approach, the technical developments are based on set invariance concepts rather than Lyapunov-based methods. This in principle provides less conservative solutions, while it might desirable in extending the method to more complicated dynamical models. 
\end{enumerate}   

Compared to our earlier work, the vector field construction presented here is not the same with the one in \cite{Panagou_Kumar_TRO14}. Furthermore, the proposed construction, coordination protocol and technical developments are not the same with the ones in \cite{PanagouCDC13}. Moreover, since it offers feedback solutions with certain convergence guarantees, it can be used as a basis in constrained model predictive control designs \cite{Panagou_ACC2013}, which are appropriate for uncertain environments. The case of mixed environments, i.e., of multiple agents operating among physical obstacles under uncertainty, are not considered in this paper and this topic is left open for future research. 

Part of this work has appeared in \cite{Panagou_ICRA14}. The current paper additionally includes: \begin{inparaenum}\item[(i)] a detailed presentation of the overall method both for the static and the dynamic case, along with the proofs which have been omitted in the conference version in the interest of space, \item[(ii)] more simulation results which demonstrate the efficacy of the method in static and dynamic environments.\end{inparaenum}

The paper is organized as follows: Section \ref{Critical Points of Vector Fields} includes a brief overview of the notions regarding the topology of two-dimensional vector fields that are used throughout the paper. Section \ref{Navigation vector fields} characterizes the singular points of our vector fields \ac{wrt} the parameter $\lambda$, while section \ref{Feedback Motion Plans} presents the blending mechanism among vector fields, the construction of the almost global feedback motion plans and the underlying control design, along with simulation results in static obstacle environments. Section \ref{Simulations} presents the extension of the method to the distributed coordination and collision-free motion of multiple agents under various sensing/communication patterns. Our conclusions and thoughts on future work are summarized in Section \ref{Conclusions}.

\section{Singular points of vector fields}\label{Critical Points of Vector Fields}
This section provides an overview of notions from vector field topology. For more information the reader is referred to \cite{Boothby, Henle, Smooth_Manifolds}.

\begin{vectorfield}
\textnormal{A vector field on an open subset $U\subset \R^n$ is a function which assigns to each point $p\in U$ a vector $X_p\in T_{p}(\R^n)$. A vector field on $\R^n$ is $C^\infty$ (smooth) if its components relative to the canonical basis are $C^\infty$ functions on $U$.}
\end{vectorfield}
\begin{integralcurve}
\textnormal{Given a $C^\infty$ vector field $X$ on $\R^n$, a curve $t\to F(t)$ defined on an open interval $J$ of $\R$ is an integral curve of $X$ if $\frac{\textrm dF}{\textrm dt}=X_{F(t)}$ on $J$.} %By definition, an integral curve is connected.}
\end{integralcurve}
\begin{singularpoint}
\textnormal{A point $p$ of $U$ at which $X_p=0$ is called a singular, or critical, point of the vector field.} % and any other point is referred to as regular.}
\end{singularpoint}

\textbf{Center-type and non-center type singularities:}
Singular points are typically distinguished to those that are reached by no integral curve (called \emph{center} type) and those that are reached by at least two integral curves (called \emph{non-center} type). In the case of a center type singularity, one can find a neighborhood of the singular point where all integral curves are closed, inside one another, and contain the singular point into their interior. In the case of non-center type singularities, one has that at least two integral curves converge to the singular point. 
The local structure of a non-center type singularity is analyzed by considering the behavior of all the integral curves which pass through the neighborhood of the singular point. This neighborhood is made of several curvilinear sectors. A curvilinear sector is defined as the region bounded by a circle $C$ of arbitrary small radius, and two integral curves, $S$ and $S'$, which both converge (for either $t\rightarrow+\infty$, or $t\rightarrow-\infty$) to the singular point. The integral curves passing through the open sector $g$ (i.e., the integral curves except for $S$, $S'$) determine the following three possible types of curvilinear sectors \cite{Tricoche}:
\begin{inparaenum}
\item [(i)] \emph{Elliptic} sectors: all integral curves begin and end at the critical point.
\item [(ii)] \emph{Parabolic} sectors: just one end of each integral curve is at the critical point.
\item [(iii)] \emph{Hyperbolic} sectors: the integral curves do not reach the critical point at all.
\end{inparaenum}
The integral curves that separate each sector from the next are called separatrixes, see also Fig. \ref{fig:isolated}. %A non-center singular point may have sectors of all three types, just one type, or a combination of those. 
%More specifically, a critical point with:
%\begin{inparaenum}
%\item only parabolic sectors is called a \emph{node}.
%\item only hyperbolic sectors is called a \emph{cross point}; saddle points are cross points with four sectors.
%\item only elliptic sectors is called a \emph{rose}; an example is the dipole.
%\end{inparaenum}
%More details on the types of curvilinear sectors can be found in \cite{Tricoche}.
\begin{figure}[h]
\centering
\includegraphics[width=0.95\columnwidth,clip]{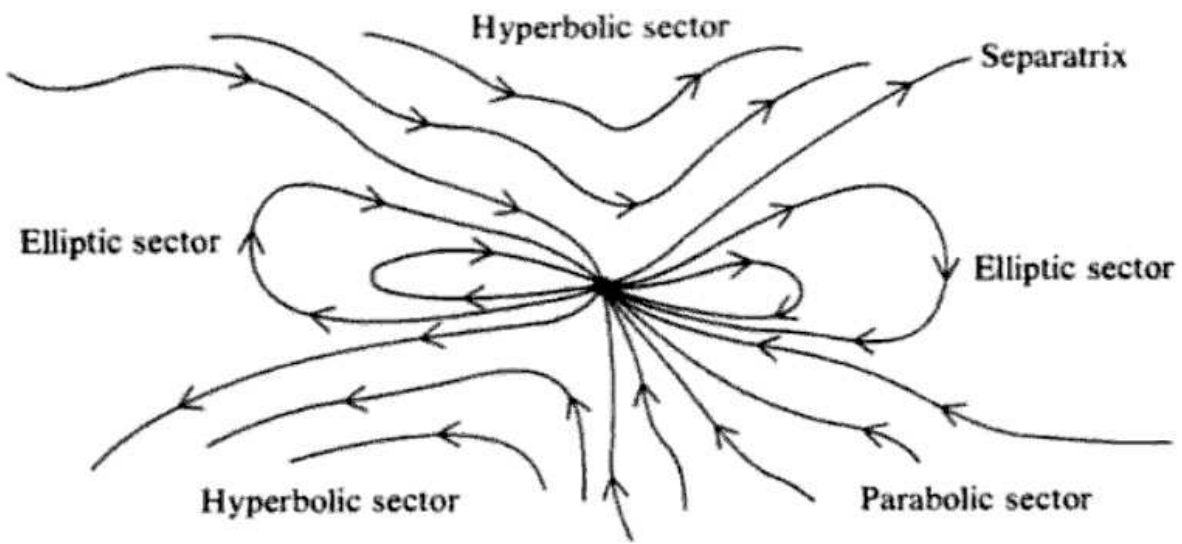}
\caption{A typical isolated critical point. Image taken from \cite{Henle}.}
\label{fig:isolated}
\end{figure}

\textbf{First-order and high-order singularities:}
A singular point $p$ of a vector field $X$ on $\R^2$ is called a \emph{first-order} singular point if the Jacobian matrix $\bm J_X(\cdot)$ of the vector field $X$ does not vanish (i.e., is nonsingular) on $p$, i.e., if: $\det\left(\bm J_X(p)\right)\neq 0$; otherwise the singular point is called \emph{high-order} singular point.%\footnote{A first-order singularity is either a saddle, a node (attractive or repelling), a focus (attractive or repelling) or a center. In this case, a complete characterization can be done by considering the eigenvalues of $\bm J_X(\bm p)$.}

\section{Navigation via vector fields}\label{Navigation vector fields}

Consider the motion of a robot with unicycle kinematics in an environment $\mathcal W$ with $N$ static obstacles. The equations of motion read:
\begin{equation}
\label{unicycle}
%\bm{\dot q}=\begin{bmatrix}\cos \theta & \sin \theta & 0 \end{bmatrix}^T u + \begin{bmatrix}0 & 0 & 1 \end{bmatrix}^T  \omega,
%\bm{\dot q}=\bm{G}(\bm{q})\bm{\nu} \Rightarrow
\begin{bmatrix}\dot x \\ \dot y \\ \dot \theta \end{bmatrix}=\begin{bmatrix}\cos \theta & 0 \\\sin \theta & 0 \\0 & 1 \end{bmatrix} \begin{bmatrix}u \\ \omega \end{bmatrix},
\end{equation}
where $\bm{q}=\left[\bm r^T\;\;\theta\right]^T$ is the configuration vector, $\bm{r}=\left[x\;\;y\right]^T$ is the position and $\theta$ is the orientation of the robot \ac{wrt} a global frame $\mathcal G$, and $u$, $\omega$ are the linear and the angular velocity of the robot, respectively.
The robot is modeled as a closed circular disk of radius $\varrho$, and each obstacle $\mathcal O_i$ is modeled as a closed circular disk of radius $\varrho_{oi}$ centered at $\bm r_{oi}=\left[x_{oi}\;\;y_{oi}\right]^T$, $i\in\{1,\dots,N\}$. Denote $\mathcal O_i=\{\bm r\in \R^2 \;|\; \|\bm r-\bm r_{oi}\|\leq \varrho_{oi}\}$. %We are interested in constructing a feedback motion plan for steering the robot from (almost) any initial collision-free configuration $\bm{q_0}$ to a goal configuration $\bm{q_g}$.

\subsection{A family of vector fields for robot navigation}

We consider the class of vector fields $\mathbf F:\R^2\rightarrow\R^2$ given by \eqref{2-d dipolar}. The vector field components $\F_x$, $\F_y$ read:
\begin{subequations}
\label{vector field components}
\begin{align}
\F_x &= (\lambda - 1) p_x x^2 + \lambda p_y x y - p_x y^2,\\
\F_y &= (\lambda - 1) p_y y^2 + \lambda p_x x y - p_y x^2.
\end{align}
\end{subequations}

\begin{critical point}
\textnormal{The origin $\bm r=\bm 0$ is the unique singular point of the vector field $\mathbf F$ \eqref{2-d dipolar} if and only if $\lambda\neq 1$.}
\end{critical point}
\begin{proof}
\textnormal{It is straightforward to verify that $\bm r=\bm 0$ is a singular point of $\mathbf F$. Let us write the vector field components \eqref{vector field components} of $\mathbf F$ in matrix form as:
\begin{align}
\label{vector valued map}
\begin{bmatrix}\F_x\\\F_y\end{bmatrix}=\underbrace{\begin{bmatrix}(\lambda-1)x^2-y^2&\lambda xy\\ \lambda xy&(\lambda-1)y^2-x^2\end{bmatrix}}_{\bm A(\lambda,\bm r)}\begin{bmatrix}p_x\\p_y\end{bmatrix}.
\end{align}
The determinant of the matrix $\bm A(\lambda,\bm r)$ is:
%\begin{align}
%\label{determinant}
$\det(\bm A(\lambda,\bm r))=-(\lambda-1)(x^2+y^2)^2.$
%\end{align}
This implies that $\bm A(\lambda,\bm r)$ is nonsingular away from the origin $\bm r=\bm 0$ if and only if $\lambda\neq1$. Therefore, for $\lambda\neq 1$ and $\bm r\neq\bm 0$, one has $\mathbf F=\bm 0$ if and only if $\bm p=\bm 0$. Since $\bm p\neq \bm 0$ by definition, if follows that the vector field $\mathbf F$ is nonsingular everywhere but the origin $\bm r=\bm 0$, as long as $\lambda\neq 1$.}
\end{proof}

\vspace{1mm}
\begin{reflection}\label{reflection}
\textnormal{The line $l: y=\tan\varphi\; x,$ where $\tan\varphi \triangleq \frac{p_y}{p_x}$, is an axis of reflection, or mirror line, for $\mathbf F$ \eqref{2-d dipolar}.}
\end{reflection}
\begin{proof}
\textnormal{Consider two points $A$, $B$ of equal distance and on opposites sides \ac{wrt} the line $l$ (Fig. \ref{fig:reflection}). Their position vectors $\bm r_A=\left[x_A\;\;y_A\right]^T$, $\bm r_B=\left[x_B\;\;y_B\right]^T$ \ac{wrt} $\mathcal G$ read:
\begin{subequations}
\begin{align}
x_A &= R\cos a,  & \quad y_A &= R\sin a\label{rA}, \\
x_B &= R\cos(2\varphi-a),  & \quad y_B &= R\sin(2\varphi-a)\label{rB},
\end{align}
\end{subequations}
where $(R,a)$, $(R,(2\varphi-a))$ are the polar coordinates of $A$, $B$, respectively.
\begin{figure}[h]
\centering
\includegraphics[width=0.95\columnwidth,clip]{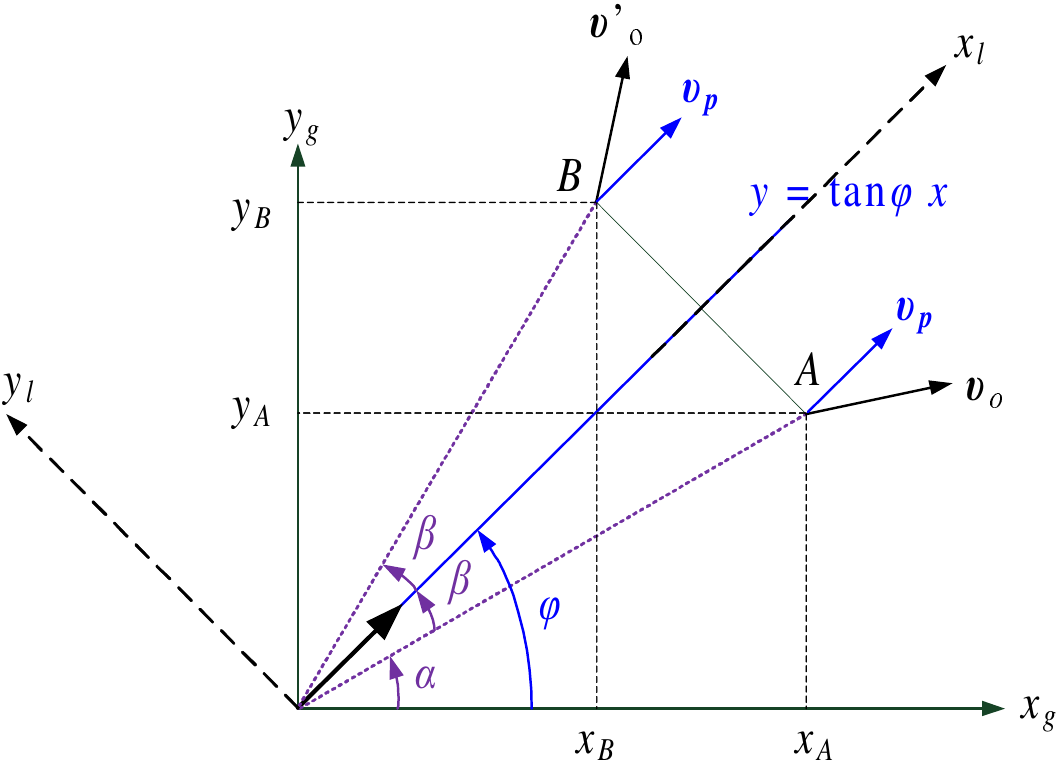}
\caption{The line $l: y=\tan\varphi \; x$, where $\varphi=\arctan(\frac{p_y}{p_x})$, is a reflection (or mirror) line for the vector field $\mathbf F$.} %This furthermore implies that the choice of the vector $\bm p=\left[p_x\;\;\;p_y\right]^T$ in \eqref{2-d dipolar} dictates the orientation of the integral curves \ac{wrt} $\mathcal G$.}
\label{fig:reflection}
\end{figure}
We need to prove that the vector $\mathbf F(\bm r_A)$, denoted $\mathbf F_A$, reflects to the vector $\mathbf F(\bm r_B)$, denoted $\mathbf F_B$, \ac{wrt} the line $l: y=\tan\varphi \; x$. Recall that the reflection matrix about the considered line $l$ is:
\begin{align}
\label{reflection matrix}
\bm H(2\varphi)=\begin{bmatrix}\cos2\varphi&\sin2\varphi\\\sin2\varphi&-\cos2\varphi\end{bmatrix}.
\end{align}
Substituting \eqref{rA} into \eqref{vector valued map} and after some standard algebra yields:
\begin{align}
\label{F_A}
\mathbf F_A&=\underbrace{\frac{(\lambda-2)R^2\|\bm p\|}{2}\left[\begin{matrix}\cos\varphi\\\sin\varphi\end{matrix}\right]}_{\bm v_p}+\underbrace{\frac{\lambda R^2\|\bm p\|}{2}\left[\begin{matrix}\cos(\varphi-2a)\\-\sin(\varphi-2a)\end{matrix}\right]}_{\bm v_o},
\end{align}
where $\|\bm p\|=\sqrt{{p_x}^2+{p_y}^2}$. Similarly, substituting \eqref{rB} into \eqref{vector valued map} yields:
\begin{align}
\label{F_B}
\mathbf F_B&=\underbrace{\frac{(\lambda-2)R^2\|\bm p\|}{2}\left[\begin{matrix}\cos\varphi\\\sin\varphi\end{matrix}\right]}_{\bm v_p}\nonumber\\
&+\underbrace{\frac{\lambda R^2\|\bm p\|}{2}\left[\begin{matrix}\cos2\varphi&\sin2\varphi\\\sin2\varphi&-\cos2\varphi\end{matrix}\right]\left[\begin{matrix}\cos(\varphi-2a)\\-\sin(\varphi-2a)\end{matrix}\right]}_{\bm v'_o}.
\end{align}
%For $\lambda=2$ it holds that $\bm v_p=\bm 0$, and then \eqref{F_A}, \eqref{F_B} yield:
%\begin{align*}
%\mathbf F\big|_{\lambda=2}(\bm r_A)&=R^2\|\bm p\|\begin{bmatrix}\cos(\varphi-2a)\\-\sin(\varphi-2a)\end{bmatrix},\\
%\mathbf F\big|_{\lambda=2}(\bm r_B)&=R^2\|\bm p\|\begin{bmatrix}\cos2\varphi&\sin2\varphi\\\sin2\varphi&-\cos2\varphi\end{bmatrix}\begin{bmatrix}\cos(\varphi-2a)\\-\sin(\varphi-2a)\end{bmatrix},\\
%&\Downarrow\\
%$\mathbf F\big|_{\lambda=2}(\bm r_A)=\bm H(2\varphi)\mathbf F\big|_{\lambda=2}(\bm r_B)$,
%\end{align*}
%which proves the argument. For $\lambda\neq 2$ 
One has: $\mathbf F_A=\bm v_p+\bm v_o$ and $\mathbf F_B=\bm v_p+\bm v'_o.$ Out of \eqref{F_A}, \eqref{F_B} one gets that $\bm v'_o=\bm H(2\varphi)\bm v_o$, i.e., $\bm v'_o$ is the reflection of the vector $\bm v_o$ about the line $l$. Thus, one may write $$\bm v_o=v^l_{ox} \hat{\bm x}_l + v^l_{oy} \hat{\bm y}_l \;\; \mbox{and} \;\; \bm v'_o=v^l_{ox} \hat{\bm x}_l - v^l_{oy} \hat{\bm y}_l,$$ where $\hat{\bm x}_l$, $\hat{\bm y}_l$ are the unit vectors along the axes $x_l$, $y_l$, respectively, see Fig. \ref{fig:reflection}. Furthermore, $\bm v_p$ is parallel to the vector $\bm p$, i.e., parallel to the candidate reflection line $l$. Consequently, one may write: $\bm v_p=v^l_{px} \hat{\bm x}_l + 0\; \hat{\bm y}_l$. It follows that:
\begin{align*}
\mathbf F_A&=(v^l_{ox}+v^l_{px}) \hat{\bm x}_l + v^l_{oy} \hat{\bm y}_l,\\
\mathbf F_B&=(v^l_{ox}+v^l_{px}) \hat{\bm x}_l - v^l_{oy} \hat{\bm y}_l,
\end{align*}
i.e., that the vector $\mathbf F_B$ is a reflection of vector $\mathbf F_A$ about the line $l$. This completes the proof.
}
\end{proof}

\begin{Remark}
\textnormal{The Jacobian matrix of $\mathbf F$ is singular at $\bm r=\bm 0$, which implies that $\bm r=\bm 0$ is a high-order singularity. Thus, one may expect that the pattern of the integral curves around the singular point will be more complicated compared to those around a first-order singularity, i.e., around nodes, saddles, foci or centers.
}
\end{Remark}

\begin{integral curves}
\textnormal{The equation of the integral curves of $\mathbf F$ for $\bm p= \left[1\;\;0\right]^T$ is given as:
\begin{align}
\label{integral curve equation}
{(x^2+y^2)}^{\frac{\lambda}{2}} = c \; y^{(\lambda-1)}, \;\; c\in \R.
\end{align}
}
\end{integral curves}

\begin{proof}
\textnormal{Consider the polar coordinates $(r\cos\phi, r\sin\phi)$ of a point $(x,y)$ where:
\begin{align}
\label{polar coordinates}
r=\sqrt{x^2+y^2}, \quad \cos\phi=\frac{x}{r}, \quad \sin\phi=\frac{y}{r}.
\end{align}
After substituting \eqref{polar coordinates} and $p_x=1$, $p_y=0$ into \eqref{vector valued map} the vector field components read:
\begin{subequations}
\label{polar vector field}
\begin{align}
\F_{x} &= r^2 \left((\lambda - 1)\cos^2\phi -\sin^2\phi\right), \\
\F_{y} &= r^2 \left(\lambda \cos\phi \sin\phi\right).
\end{align}
\end{subequations}
An integral curve of \eqref{2-d dipolar} is by definition the solution of the system of ordinary differential equations:
\begin{align}
\label{integral curve}
\begin{matrix}
\frac{dx}{dt}=\F_x \\
\frac{dy}{dt}=\F_y
\end{matrix}, \quad \mbox{which further reads:} \quad  \frac{dx}{dy}=\frac{\F_x}{\F_y},
\end{align}
while the differentials between Cartesian and polar coordinates satisfy the formula:
\begin{align}
\label{differentials}
\begin{bmatrix}dr\\rd\phi\end{bmatrix}=\begin{bmatrix}\cos\phi&\sin\phi\\-\sin\phi&\cos\phi\end{bmatrix}\begin{bmatrix}dx\\dy\end{bmatrix}.
\end{align}
Plugging \eqref{differentials}, \eqref{polar vector field} into \eqref{integral curve} results in:
\begin{align*}
\frac{1}{r}\;dr&=(\lambda-1)\frac{\cos\phi}{\sin\phi}\;{d\phi},
\end{align*}
while integrating by parts yields:
\begin{align*}
\ln(r)&=(\lambda-1)\ln(\sin\phi)+\ln(c) \Rightarrow \\
\ln(r)&=\ln\left(c\;\sin^{(\lambda-1)}\phi\right) \Rightarrow \\
r&= c \;\sin^{(\lambda-1)}\phi \Rightarrow r= c \; \frac{y^{(\lambda-1)}}{r^{(\lambda-1)}} \Rightarrow \\
r^\lambda&= c \; y^{(\lambda-1)} \Rightarrow {(x^2+y^2)}^{\frac{\lambda}{2}} = c \; y^{(\lambda-1)}, \; \mbox{where} \; c\in\R.
\end{align*}
This completes the proof.}
\end{proof}

\vspace{1mm}
\begin{Remark}
It is straightforward to verify that:
\begin{itemize}
\item For $\lambda=0$, \eqref{integral curve equation} reduces to $y=c$, i.e., the integral curves are straight lines parallel to $\bm p=\left[1\;\;\;0\right]^T$.
\item For $\lambda=1$, \eqref{integral curve equation} reduces to $\sqrt{x^2+y^2}=c$, i.e., the integral curves are circles of radius $\sqrt{c}$, where $c>0$, centered at the origin $(x,y)=(0,0)$.
\end{itemize}
\end{Remark}

\subsection{Attractive vector fields}
%The phase portrait of a vector field is a standard means of studying the behavior of the integral curves in the vicinity of a singular point. 
Let us consider the case $\lambda=2$. Take for simplicity $\bm p=\left[1\;\;0\right]^T$ and write the vector field components as:
\begin{subequations}
\label{nominal dipole}
\begin{align}
\F_x &= x^2 - y^2,\\
\F_y &= 2 x y
\end{align}
\end{subequations}
Following \cite{Henle}, the singular point $\bm r=\bm 0$ of \eqref{nominal dipole} is a dipole. More specifically, the vector field \eqref{nominal dipole} has two elliptic sectors, with the axis $y=0$ serving as the separatrix. This implies that all integral curves \emph{begin and end at} the singular point, \emph{except for} the separatrix $y=0$. The separatrix converges to $\bm r=\bm 0$ for $x<0$ and diverges for $x>0$ (Fig. \ref{fig:dip2}). Out of Theorem \ref{reflection}, the separatrix $y=0$ is the reflection line for the vector field \eqref{nominal dipole}.
\begin{figure}[h]
\centering
\includegraphics[width=0.7\columnwidth,clip]{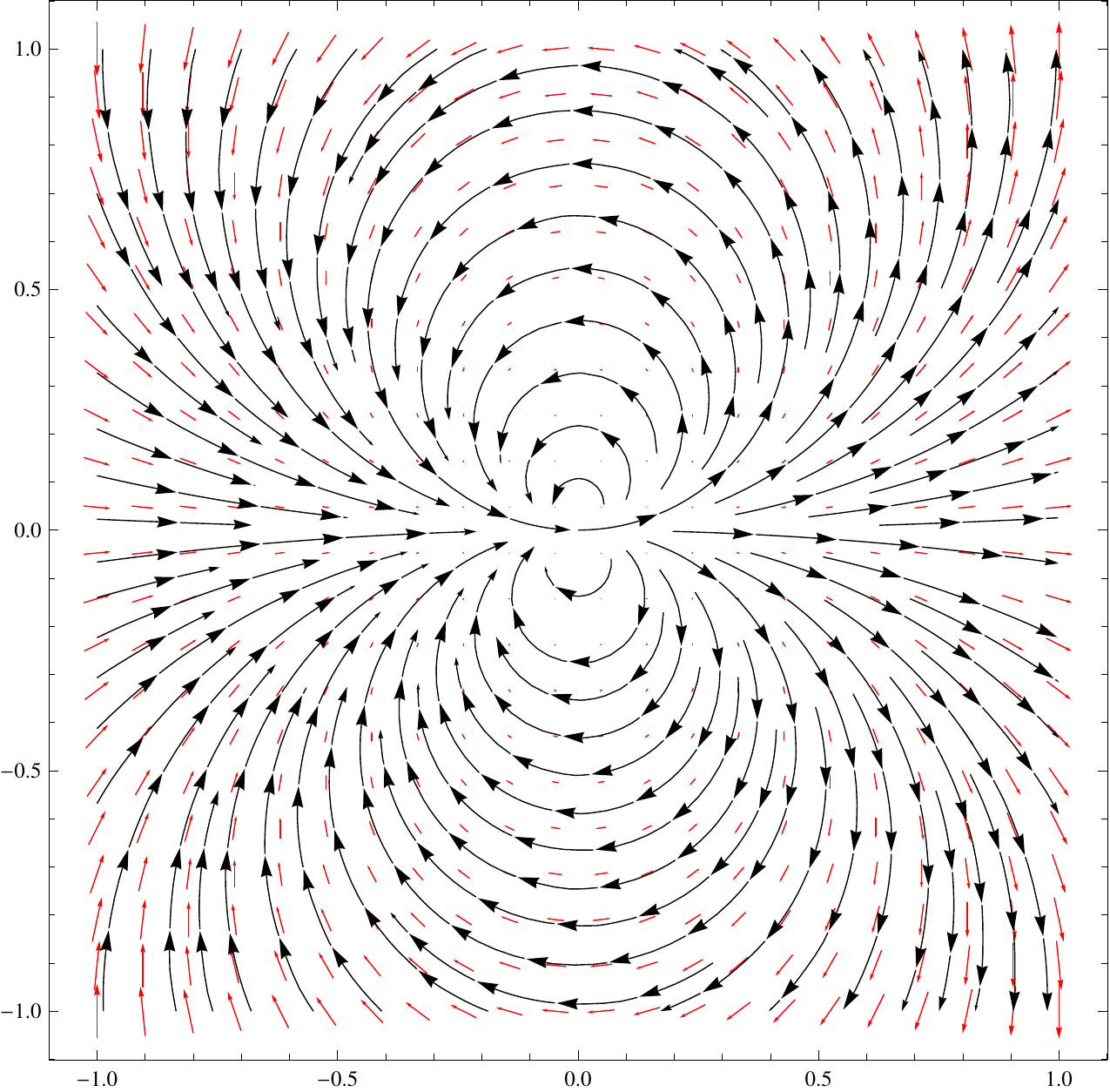}
\caption{The integral curves of \eqref{2-d dipolar} for $\lambda=2$, $p_x=1$, $p_y=0$.} %Note that the vector $\bm p$ is always a reflection line for $\mathbf F$, while for $\lambda=2$ the singular point $\bm r=\bm 0$ is of dipole type.}
\label{fig:dip2}
\end{figure}

Furthermore, Theorem \ref{reflection} implies that the axis the vector $\bm p\neq \bm 0$ lies on is, in general, a reflection line for \eqref{2-d dipolar}. This means that the resulting integral curves are symmetric \ac{wrt} the vector $\bm p\in\R^2$. 
In that sense, any of the integral curves of $\mathbf F$ offers a path to $\bm r=\bm 0$, while at the same time the direction of the vector $\bm p$ dictates the symmetry axis of the integral curves \ac{wrt} the global frame $\mathcal G$.

\vspace{1mm}
Therefore, defining a feedback motion plan for steering the unicycle to a goal configuration $\bm q_g=\left[\bm r_g^T \;\;\theta_g\right]^T$ has been based in earlier work of ours' \cite{Panagou_CDC11A} on the following simple idea:
%\begin{inparaenum}
%\item
Pick a vector field $\mathbf F$ out of \eqref{2-d dipolar} in terms of $(\bm r-\bm r_g)$,\footnote{This is to have the unique singular point of $\mathbf F$ coinciding with the desired position $\bm r_g$.} with $\lambda = 2$ and $\bm p=\left[p_x\;\;p_y\right]^T$, so that the direction of the vector $\bm p$ coincides with the goal orientation: $\varphi \triangleq \arctan(\frac{p_y}{p_x})= \theta_g$. Then, the integral curves serve as a reference to steer the position trajectories $\bm r(t)$ to the goal position $\bm r_g$, and the orientation trajectories $\theta(t)$ to the goal orientation $\theta_g$.%\footnote{Actually, one may verify out of the phase portrait of $\mathbf F$ that the critical point $\bm r=\bm 0$ is a dipole for any $\lambda>1$.}
%\item Force the system vector field $\bm{\dot q}\in T_{\bm q}\mathcal Q$ to ``align with" and ``flow along" the integral curves of $\mathbf F$, until reaching the origin $\bm q=\bm 0$.%\footnote{The control design in \cite{Panagou_CDC11A} achieves global asymptotic stabilization of the system trajectories $\bm q(t)$ to the origin, with the solutions understood in the Filippov sense.}
%\end{inparaenum}

%\vspace{1mm}
%In the sequel, we call the class of vector fields \eqref{2-d dipolar} for $\lambda=2$ as \textbf{attractive} vector fields (or flows) to a configuration $\bm q\in \R^2\times \mathcal S$.

\subsection{Repulsive vector fields}
Let us consider the case $\lambda=1$, i.e., the case when the vector field \eqref{2-d dipolar} has multiple singular points. The vector field components read:
\begin{subequations}
\label{multiple singular points}
\begin{align}
\F_x &= p_y x y - p_x y^2,\\
\F_y &= p_x x y - p_y x^2.
\end{align}
\end{subequations}
The vector field \eqref{multiple singular points} vanishes on the set $\mathcal V=\{\bm r\in \R^2 \; | \; p_y x - p_x y=0\}.$ Out of Theorem \ref{reflection}, the singularity set $\mathcal V$ coincides with the reflection line of the vector field \eqref{multiple singular points}. The equation of the integral curves can be computed for $p_y x - p_x y\neq 0$ as: $\frac{dx}{dy}=\frac{y}{-x}\Rightarrow x^2+y^2=c^2$, where $c\in \R$, which implies that the integral curves are circles centered at the origin $\bm r=\bm 0$, see Fig. \ref{fig:multiple}. %To visualize the behavior of the integral curves, take $\bm p=\left[1\;\;0\right]^T$ to get the vector field depicted in Fig. \ref{fig:multiple}.
\begin{figure}[h]
\centering
\includegraphics[width=0.7\columnwidth,clip]{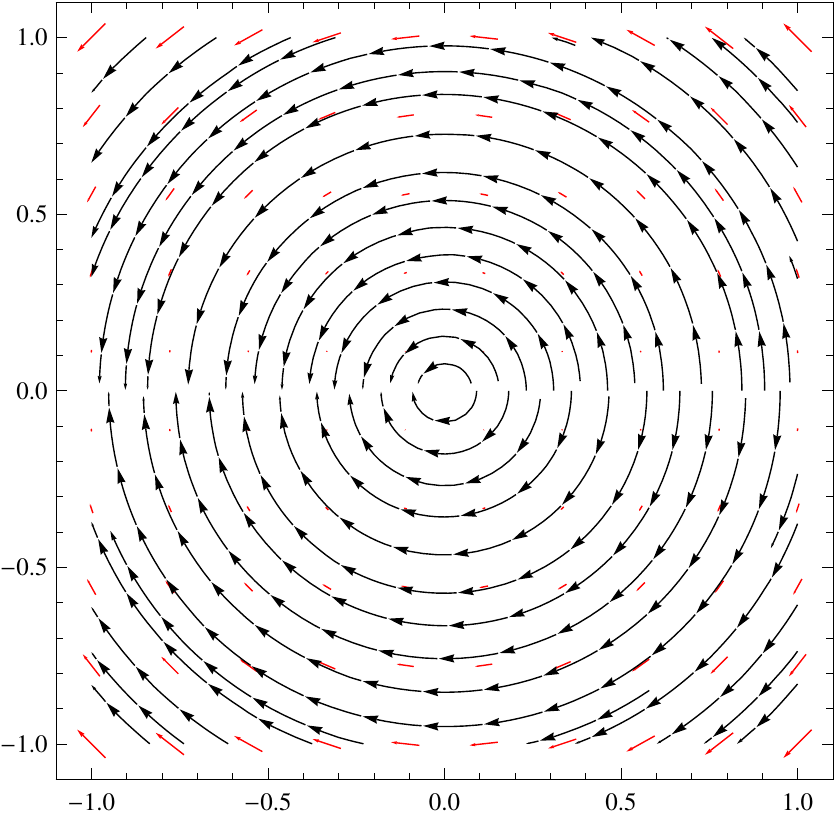}
\caption{The vector field $\mathbf F$ for $\lambda=1$ and $p_x=1$, $p_y=0$.}
\label{fig:multiple}
\end{figure}

\vspace{1mm}
The signum of $x$ (in general, of ${\bm p_i}^T\bm r$) dictates whether the integral curves escape the singularity set $\mathcal V$ (see the half-plane $x>0$) or converge to the singularity set $\mathcal V$ (see the half-plane $x<0$). We say that the singular point $\bm r=\bm 0$ of the vector field \eqref{multiple singular points} is of center type; this means that \emph{no integral curve reaches the singular point.}\footnote{Characterizing this particular singularity as of center type is slightly inconsistent with standard notation, since in this case the singular point $\bm r=\bm 0$ is not isolated.}

%The departing behavior of the integral curves away from the singularity set $\mathcal V$ resembles the pattern of a fluid flow around a cylinder. 
Thus, one may employ \eqref{multiple singular points} to define tangential vector fields locally around circular obstacles.

%\vspace{1mm}
%In the sequel we call the class of vector fields \eqref{2-d dipolar} for $\lambda=1$ as \textbf{repulsive} vector fields (or flows) around a point $\bm r\in \R^2$.

\section{Almost global feedback motion plans}\label{Feedback Motion Plans}

Given the class of attractive and repulsive vector fields, the idea on defining an almost global feedback motion plan $\mathbf F^\star$ on the collision-free space $\mathcal F$ is now simple: we pursue to combine an attractive-to-the-goal vector field $\mathbf F_g$ with (local) repulsive vector fields $\mathbf F_{oi}$ around each obstacle $\mathcal O_i$, so that the integral curves of $\mathbf F^\star$:
\begin{inparaenum}
\item converge to the goal $\bm q_g$, and
\item point into the interior of $\mathcal F$ on the boundaries of the obstacles $\mathcal O_i$.
\end{inparaenum}
The vector field $\mathbf F^\star$ can then serve as a feedback motion plan on $\mathcal W$.

\vspace{1mm}
\begin{Remark}
Combining the vector fields $\mathbf F_g$, $\mathbf F_{oi}$ should be done carefully so that the resulting vector field $\mathbf F^\star$ does not have any undesired singularities on $\mathcal F$. For this reason, we consider the normalized unit vector fields:
\begin{subequations}
\begin{align}
\mathbf F^n_g=\left\{
                \begin{array}{cc}
                  \frac{\mathbf F_g}{\|\mathbf F_g\|}, & \mbox{ for } \; \bm r\neq\bm 0; \\
                  \bm 0, & \mbox{ for } \bm r=\bm 0.
                \end{array}
              \right.\label{normalized attractive}\\
\mathbf F^n_{oi}=\left\{
                \begin{array}{cc}
                 \frac{\mathbf F_{oi}}{\|\mathbf F_{oi}\|}, & \mbox{ for } \; \bm r\notin\mathcal V_i; \\
                  \bm 0, & \mbox{ for } \bm r\in\mathcal V_i.
                \end{array}
              \right.
\end{align}
\end{subequations}
respectively, when defining the blending mechanism, see later on in Section \ref{blending}.
\end{Remark}

\subsection{Attractive vector field to the goal}
Without loss of generality we assume that $\bm q_g= \bm 0$. An attractive-to-the-goal vector field $\mathbf F_g$ may be taken out of \eqref{2-d dipolar} for $\lambda=2$, $\bm p_g=\left[1\;\;0\right]^T$, which yields the vector field \eqref{nominal dipole}. The components of the normalized vector field $\mathbf F^n_g$ taken out of \eqref{normalized attractive} for $x\neq0$, $y\neq0$ read: $$\F^n_{gx} = \frac{x^2 - y^2}{x^2+y^2}, \quad \F^n_{gy} = \frac{2 x y}{x^2+y^2}.$$

\subsection{Repulsive vector field \ac{wrt} a circular obstacle}
Consider an obstacle $\mathcal O_i$ and the region %$\mathcal Z_i$, modeled as the closed circular disk of radius $\varrho_{\mathcal Zi}$ centered at $\bm r_{oi}$, 
$\mathcal Z_i : \left\{\bm r\in \R^2 \; | \; \|\bm r-\bm r_{oi}\|\leq \varrho_{\mathcal Zi}\right\}$, where $\varrho_{\mathcal Zi}=\varrho_{oi}+\varrho+\varrho_{\varepsilon}$, see Fig. \ref{fig:repulsive}. The parameter $\varrho_{\varepsilon}\geq 0$ is the minimum distance that the robot is allowed to keep \ac{wrt} the boundary of the obstacle.
\begin{figure}
\centering
\includegraphics[width=0.715\columnwidth,clip]{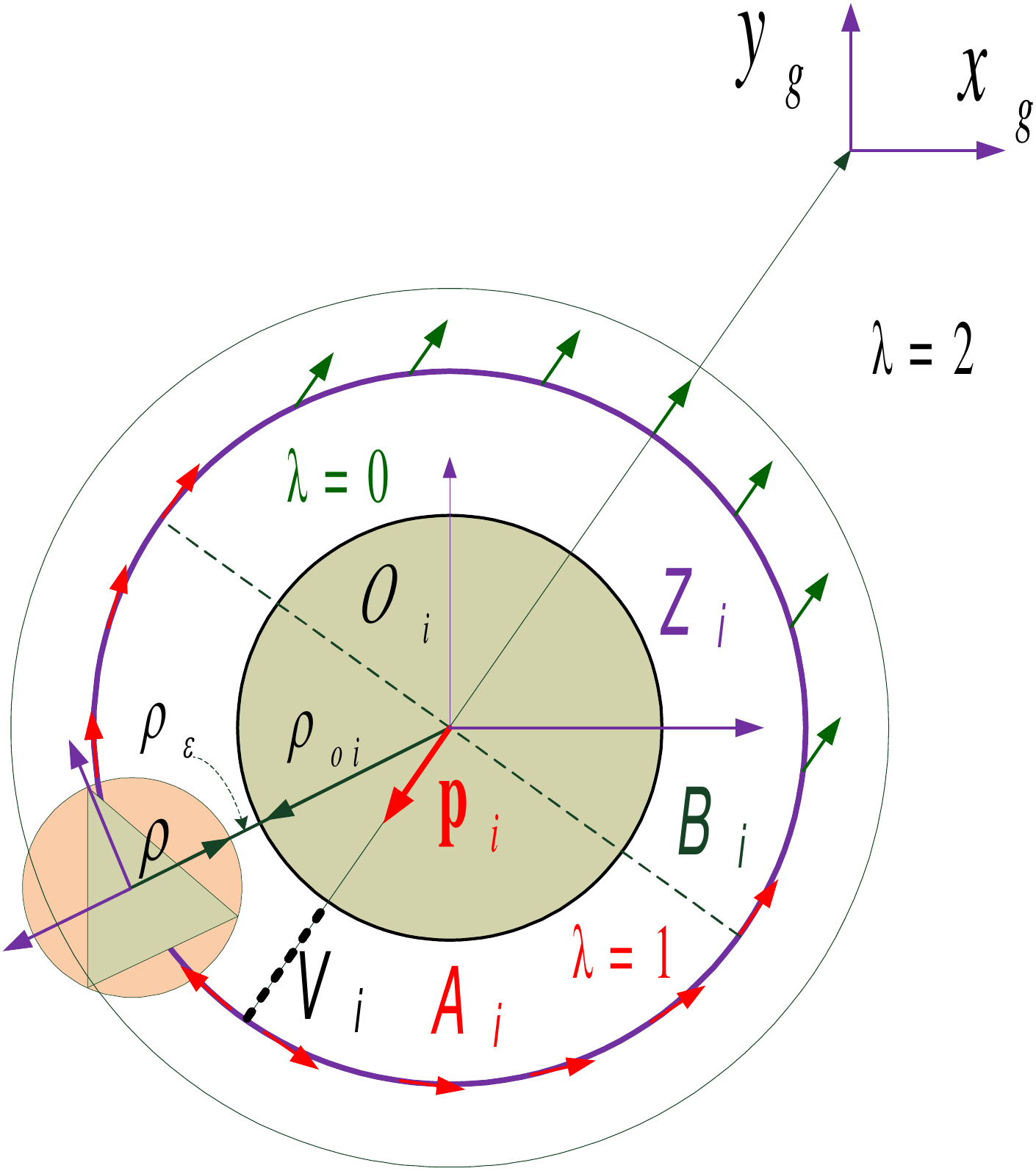}
\caption{Defining a repulsive vector field $\mathbf F_{oi}$ around the obstacle $\mathcal O_i$. Note that we take the vector field \eqref{2-d dipolar} with $\lambda=1$ in region $\mathcal A_i$ and with $\lambda=0$ in region $\mathcal B_i$.}
\label{fig:repulsive}
\end{figure}

A repulsive vector field \ac{wrt} the point $\bm r_{oi}$ can be picked out of \eqref{multiple singular points} for $\bm p_i=\left[p_{xi}\;\;p_{yi}\right]^T$, where $p_{xi}=\cos\phi_i$, $p_{yi}=\sin\phi_i$, $\phi_i=\atan2(-y_{oi},-x_{oi})+\pi$ as:
\begin{subequations}
\label{repulsive}
\begin{align}
\F_{oxi} &= p_{yi} (x-x_{oi}) (y-y_{oi}) - p_{xi} (y-y_{oi})^2,\\
\F_{oyi} &= p_{xi} (x-x_{oi}) (y-y_{oi}) - p_{yi} (x-x_{oi})^2.
\end{align}
\end{subequations}
Note that the vector $\bm p_i$ is picked such that it lies on the line connecting the center $\bm r_{oi}$ of the obstacle with the goal point $\bm r_g=\bm 0$. Therefore, the singularity set $\mathcal V_i$ of \eqref{repulsive} lies by construction on this line, which is also the reflection axis of the vector field \eqref{repulsive}.
Denote $\mathcal A_i=\{\bm r \in \mathcal Z_i \; | \; {\bm p_i}^T (\bm r-\bm r_{oi})\geq0\}$, $\mathcal B_i=\{\bm r \in\mathcal Z_i \; | \; {\bm p_i}^T (\bm r-\bm r_{oi})<0\}$, where $\mathcal Z_i=\mathcal A_i \bigcup \mathcal B_i$, and consider the behavior of the integral curves around the singularity set $\mathcal V_i$. The integral curves depart from the singularity set $\mathcal V_i$ in the region $\mathcal A_i$ (see the red vectors around $\mathcal V_i$ in Fig. \ref{fig:repulsive}), and converge to the singularity set $\mathcal V_i$ in the region $\mathcal B_i$ (the corresponding vectors have \emph{not} been drawn in Fig. \ref{fig:repulsive}). The integral curves in region $\mathcal A_i$ render safe, tangential reference paths around the obstacle $\mathcal O_i$. However, their pattern in region $\mathcal B_i$ is undesirable, since it may trap the system trajectories $\bm r(t)$ away from $\bm r_g$. To overcome this, in region $\mathcal B_i$ we define a vector field out of \eqref{2-d dipolar} for $\lambda=0$ and $\bm p_{i}$ as before, whose vector field components read:
\begin{subequations}
\label{repulsive 2}
\begin{align}
\F_{oxi} &= -p_{xi} (x-x_{oi})^2 - p_{xi} (y-y_{oi})^2,\\
\F_{oyi} &= -p_{yi} (x-x_{oi})^2 - p_{yi} (y-y_{oi})^2.
\end{align}
\end{subequations}
This vector field is co-linear with $\bm p_i$ and vanishes at the unique singular point $\bm r=\bm r_{oi}$.
\vspace{1mm}
\begin{Remark}
The transition of the integral curves between regions $\mathcal A_i$, $\mathcal B_i$ is smooth, since the vectors at the points where ${\bm p_i}^T (\bm r-\bm r_{oi})=0$ coincide.
\end{Remark}

In summary, the vector field $\mathbf F_{oi}$ around a circular obstacle $\mathcal O_i$ is picked out of the family of vector fields \eqref{2-d dipolar} as:
\begin{align}
\label{repulsive full}
\mathbf F_{oi}=\left\{
                 \begin{array}{cc}
                   \mathbf F_{(\lambda=1)}\left(\bm\delta\bm r_i\right), & \; \hbox{for}\;\; {\bm p_i}^T (\bm\delta\bm r_i)\geq0;\\
                   \mathbf F_{(\lambda=0)}\left(\bm\delta\bm r_i\right), & \; \hbox{for}\;\; {\bm p_i}^T (\bm\delta\bm r_i)<0,
                 \end{array}
               \right.
\end{align}
where $\bm\delta\bm r_i\triangleq \bm r-\bm r_{oi}$, $\phi_i\triangleq\atan2(-y_{oi},-x_{oi})+\pi$, $\bm p_i=\left[\cos\phi_i\;\;\sin\phi_i\right]^T$. The normalized vector field then reads:
\begin{align}
\label{repulsive normalized full}
\mathbf F^n_{oi}=
\left\{
\begin{array}{cccc}
\frac{\mathbf F_{(\lambda=0)}\left(\bm\delta\bm r_i\right)}{\|\mathbf F_{(\lambda=0)}\left(\bm\delta\bm r_i\right)\|}, & \; \hbox{for}\;\; {\bm p_i}^T (\bm\delta\bm r_i)<0;\\
\frac{\mathbf F_{(\lambda=1)}\left(\bm\delta\bm r_i\right)}{\|\mathbf F_{(\lambda=1)}\left(\bm\delta\bm r_i\right)\|}, & \; \hbox{for}\;\; {\bm p_i}^T (\bm\delta\bm r_i)\geq0 \mbox{ and } \bm r\notin \mathcal V_i;\\
\bm 0, & \; \hbox{for}\;\; {\bm p_i}^T (\bm\delta\bm r_i)\geq0 \mbox{ and } \bm r\in \mathcal V_i;
\end{array}\right.
\end{align}

\subsection{Blending attractive and repulsive vector fields}\label{blending}
%The effect of the repulsive vector field $\mathbf F^n_{oi}$ should be restricted in the region $\mathcal Z_i$, while away from $\mathcal Z_i$ the effect of the attractive vector field $\mathbf F^n_g$ should drive the robot to the goal. To encode this requirement, 
Define the obstacle function $\beta_i(\cdot):\R^2\rightarrow\R$ as:
\begin{align}
\label{obstacle constraint}
\beta_i(\bm r,\bm r_{oi}, \varrho_{oi})={\varrho_{oi}}^2-\|\bm r-\bm r_{oi}\|^2,
%(x-x_{oi})^2-(y-y_{oi})^2,
\end{align}
which is positive in the interior $\mathrm{Int}(\mathcal O_i)$ of the obstacle, zero on the boundary $\partial\mathcal O_i$ of the obstacle, and negative everywhere else. Denote the value of the constraint function $\beta_i$ on the boundary $\partial \mathcal Z_i$ of the region $\mathcal Z_i$ as $\beta_{i\mathcal Z}=- 2{\varrho _{oi}}\left( {\varrho  + {\varrho _\varepsilon }} \right) - {\left( {\varrho  + {\varrho _\varepsilon }} \right)^2}.$ 

The repulsive vector field $\mathbf F^n_{oi}$ is then locally defined on the set: $\left(\mathcal Z_i\setminus\mathrm{Int}(\mathcal O_i)\right)=\{\bm r\in\R^2 \;|\; \beta_{i\mathcal Z}\leq \beta_i(\bm r)\leq0\}.$ At the same time, the attractive vector field $\mathbf F^n_g$ should be defined exterior to $\mathcal Z_i$, i.e., for $\beta_i(\bm r)<\beta_{i\mathcal Z}$. To encode this, define the smooth bump function $\sigma_i(\cdot):\R^2\rightarrow[0,1]$:
\begin{align}
\label{bump function}
\sigma_i = \left\{
              \begin{array}{ll}
                1, & \; \hbox{for }\;\beta_i(\bm r)\leq \beta_{i\mathcal F}; \\
                a{\beta_i}^3 + b{\beta_i}^2 + c\beta_i + d, & \; \hbox{for }\;\beta_{i\mathcal F}<\beta_i(\bm r)<\beta_{i\mathcal Z}; \\
                0, & \; \hbox{for }\;\beta_{i\mathcal Z}\leq\beta_i(\bm r);
              \end{array}
            \right.
\end{align}
where $\beta_{i\mathcal Z}$ is the value of \eqref{obstacle constraint} at distance $\varrho_{\mathcal Zi}$ \ac{wrt} $\bm r_{oi}$, $\beta_{i\mathcal F}$ is the value of \eqref{obstacle constraint} at some distance $\varrho_{\mathcal Fi}>\varrho_{\mathcal Zi}$ \ac{wrt} $\bm r_{oi}$, and the coefficients $a$, $b$, $c$ and $d$ are computed as:
\begin{align*}
a &= \frac{2}{(\beta_{i\mathcal Z}-\beta_{i\mathcal F})^3}, & b &= -\frac{3(\beta_{i\mathcal Z}+\beta_{i\mathcal F})}{(\beta_{i\mathcal Z}-\beta_{i\mathcal F})^3},\\
c &= \frac{6\beta_{i\mathcal Z}\beta_{i\mathcal F}}{(\beta_{i\mathcal Z}-\beta_{i\mathcal F})^3}, & d &= \frac{{\beta_{i\mathcal Z}}^2(\beta_{i\mathcal Z}-3\beta_{i\mathcal F})}{(\beta_{i\mathcal Z}-\beta_{i\mathcal F})^3},
\end{align*}
so that \eqref{bump function} is a $\mathcal C^2$ function. Having this at hand, and inspired by \cite{Lindemann_Lavalle_IJRR09}, one may now define the vector field:
\begin{align}
\label{one obstacle}
\mathbf F_i = \sigma_i \mathbf F^n_g + (1-\sigma_i) \mathbf F^n_{oi}.
\end{align}
\begin{1obstacle}\label{1obstacle}
\textnormal{The vector field \eqref{one obstacle} is:
\begin{enumerate}
\item[(\emph{i})] Attractive to the goal $\bm q_g$ for $\|\bm r-\bm r_{oi}\|\geq \varrho_{\mathcal Fi}$, i.e., for $\beta_{i}(\bm r)\leq\beta_{\mathcal Fi}$ where $\sigma_i=1$, via the effect of $\mathbf F^n_g$.
\item[(\emph{ii})] Repulsive \ac{wrt} $\mathcal O_i$ for $\varrho_{oi}\leq \|\bm r-\bm r_{oi}\|\leq \varrho_{\mathcal Zi}$, i.e., for $\beta_{\mathcal Zi}\leq \beta_i(\bm r)$ where $\sigma_i=0$, via the effect of $\mathbf F^n_{oi}$.
\item[(\emph{iii})] Nonsingular in the region $\varrho_{\mathcal Zi} < \|\bm r-\bm r_{oi}\| < \varrho_{\mathcal Fi}$, i.e., for $\beta_{\mathcal Fi} < \beta_i(\bm r) < \beta_{\mathcal Zi}$ where $0<\sigma_i<1$.
\item[(\emph{iv})] Safe \ac{wrt} the obstacle $\mathcal O_i$ and convergent to the goal $\bm q_g$ for almost all initial conditions.
\end{enumerate}
\vspace{1mm}
\begin{proof}
The first two arguments have been proved in the previous section. To verify the third argument, consider the norm of vector field $\mathbf F_i$ in the blending region $\mathcal D_i:\{\bm r\in\R^2 \;|\; \varrho_{\mathcal Zi} < \|\bm r-\bm r_{oi}\| < \varrho_{\mathcal Fi}\}$, which reads: %$\|\mathbf F_i\|=$
%\begin{align*}
%&=\sqrt{\left(\sigma_i \F^n_{gx}+(1-\sigma_i)\F^n_{oxi}\right)^2+\left(\sigma_i \F^n_{gy}+(1-\sigma_i)\F^n_{oyi}\right)^2}\\
%&=\sqrt{\sigma_i^2 +(1-\sigma_i)^2+2\sigma_i(1-\sigma_i)\left[\begin{smallmatrix}\F^n_{gx}&\F^n_{gy}\end{smallmatrix}\right]\left[\begin{smallmatrix}\F^n_{oxi}\\ \F^n_{oyi}\end{smallmatrix}\right]}\\
$$\|\mathbf F_i\|=\sqrt{1-2\sigma_i(1-\sigma_i)+2\sigma_i(1-\sigma_i)\cos\alpha},$$
%\end{align*}
where $\alpha$ the angle between the vectors $\mathbf F^n_g$, $\mathbf F^n_{oi}$ at some point $\bm r\in\mathcal D_i$. Then, for $\bm r \notin \mathcal V_i$ one has that $\|\mathbf F_i\|$ vanishes at the points where $\sigma_i$ is the solution of: $$2(1-\cos\alpha){\sigma_i}^2-2(1-\cos\alpha)\sigma_i+1=0.$$ The discriminant reads $\Delta=-4(1-\cos\alpha)^2$, which implies that there are no real solutions, i.e., that the vector field $\mathbf F_i$ is nonsingular for $\bm r \notin \mathcal V_i$. Moreover, for $\bm r\in \mathcal V_i$ one has $\mathbf F^n_{oi}=\bm 0$, and therefore: $\|\mathbf F_i\|=\sigma_i\neq0$. %Consequently, there are no singularities in $\mathcal D_i$. %the blending region $\beta_{\mathcal Zi}< \beta_i < \beta_{\mathcal Fi}$.
\vspace{1mm}\\
Finally, to verify the fourth argument, consider first that the integral curves which do not intersect with the blending region $\mathcal D_i$ are convergent by construction to $\bm r_g$. Consider now the boundary 
$$S_i : \{\bm r\in \R^2 \; | \; \|\bm r-\bm r_{oi}\|^2-{\varrho_{\mathcal Fi}}^2=0\}$$ 
of the region $\mathcal D_i$ and let us analyze the behavior of the integral curves on the manifolds:
\begin{align*}
S_i^-&:\{\bm r\in \R^2 \; | \; \|\bm r-\bm r_{oi}\|={\varrho_{\mathcal Fi}}+\delta\varrho\},\\
S_i^+&:\{\bm r\in \R^2 \; | \; \|\bm r-\bm r_{oi}\|={\varrho_{\mathcal Fi}}-\delta\varrho\},
\end{align*} 
with $\delta\varrho>0$ arbitrarily small. After some calculations: 
\begin{align*}
\nabla S_i \; \mathbf F_i&=2(\bm r-\bm r_{oi})^T \mathbf F^n_g,\\ 
\nabla S_i^- \; \mathbf F_i&=2(\bm r-\bm r_{oi})^T \mathbf F^n_g.
\end{align*}
%\begin{align}
%\nabla S_i \; \mathbf F_i%&=2\left[x-x_{oi}\;\;y-y_{oi}\right]\left(\sigma_i\mathbf F^n_g+(1-\sigma_i)\mathbf F^n_{oi}\right)\nonumber\\
%%&=2\sigma_i(\bm r-\bm r_{oi})^T \mathbf F^n_g + 2(1-\sigma_i)(\bm r-\bm r_{oi})^T \mathbf F^n_{oi}\nonumber\\
%&\overset{(\sigma_i=1)}{=}2(\bm r-\bm r_{oi})^T \mathbf F^n_g\label{gradSi},\\
%\nabla S_i^- \; \mathbf F_i&=2(\bm r-\bm r_{oi})^T \mathbf F^n_g\label{gradSi-}.
%\end{align}
For $\nabla S_i^+ \; \mathbf F_i$, consider the following cases: 
\begin{enumerate}
\item[Case 1.] The vector field $\mathbf F^n_{oi}$ satisfies: $(\bm r-\bm r_{oi})^T \mathbf F^n_{oi}=0,$ and therefore: $$\nabla S_i^+ \; \mathbf F_i=2\sigma_i(\bm r-\bm r_{oi})^T \mathbf F^n_g.$$
%\begin{align}
%\label{gradSi+Ai}
%\nabla S_i^+ \; \mathbf F_i=2\sigma_i(\bm r-\bm r_{oi})^T \mathbf F^n_g.
%\end{align}
%Out of \eqref{gradSi-}, \eqref{gradSi+Ai} one gets that: 
Then: $(\nabla S_i^- \; \mathbf F_i)(\nabla S_i^+  \;\mathbf F_i)>0,$ which implies that the integral curves cross the switching surface $S_i$ and enter $\mathcal A_i$. Consider now the behavior of the integral curves in $\mathcal A_i$. Assume that $\nabla S_i^+ \; \mathbf F_i=2\sigma_i(\bm r-\bm r_{oi})^T \mathbf F^n_g>0$; this would imply that $\nabla S_i \; \mathbf F_i>0$ as well, i.e., that the integral curves did not cross $S_i$, a contradiction. Then: 
$$\nabla S_i^+ \; \mathbf F_i=2\sigma_i(\bm r-\bm r_{oi})^T \mathbf F^n_g<0,$$ which yields that the integral curves approach the boundary $$T_i : \{\bm r\in \R^2 \; | \; \|\bm r-\bm r_{oi}\|^2-{\varrho_{\mathcal Zi}}^2=0\}$$ of the blending region $\mathcal D_i$. Denote $$T_i^-:\{\bm r\in \R^2 \; | \; \|\bm r-\bm r_{oi}\|={\varrho_{\mathcal Zi}}+\delta\varrho\}$$ %and $T_i^+=\{\bm r\in \R^2 \; | \; \|\bm r-\bm r_{oi}\|<{\varrho_{\mathcal Zi}}\}$,
and note that: $\nabla T_i^- \; \mathbf F_i=\nabla S_i^+ \; \mathbf F_i<0$, and that $\nabla T_i \; \mathbf F_i=0$, since on $T_i$ one has $\sigma_i=0$. Then, $\mathbf F_i\neq \bm 0$ is tangent to $T_i$, which means that the integral curves slide along $T_i$, until reaching region $\mathcal B_i$.
\begin{Remark}
The integral curves are not defined on the (unique) point on $T_i$ where $\mathbf F_i=\bm 0$. This further implies that system trajectories which either start or reach this point get stuck away from the goal configuration.
\end{Remark}
\emph{Let us now consider the pattern of the integral curves in the vicinity of the singularity and characterize the set of initial conditions from which the system trajectories end there.} It was shown in the previous section that the integral curves around the singularity set $\mathcal V_i$ are departing the set, except for one integral curve which converges to $\mathcal V_i$. For this condition to occur \emph{the goal orientation $\theta_g$ should be co-linear with the line the singularity set $\mathcal V_i$ lies on}. To see why, recall that the vector field in the blending region reads: $\mathbf F_i=\sigma_i \mathbf F^n_{g}$, and that the vector field $\mathbf F^n_{g}$ should point to the singularity set $\mathcal V_i$. Consequently, this condition arises if and only if the obstacle is positioned such that the direction of the vector $\bm p_{i}$ coincides with the direction of the vector $\bm p_g$. %\footnote{For instance, if the goal orientation is $\theta_g=0$, then this condition occurs around an obstacle which is centered at a point $(x_{oi},0)$, while the set of initial conditions from which the integral curves converge to the singularity is the $x$-axis, i.e., a lower dimensional manifold.} 
\emph{Therefore, the set of initial conditions from which the integral curves of $\mathbf F_i$ converge to the singularity set $\mathcal V_i$ is of Lebesgue measure zero.} Note also that if the direction of $\bm p_g$ does not coincide with the direction of $\bm p_i$, %then:
%\begin{inparaenum}
%\item if $\varrho_\varepsilon\neq 0$, 
then the singular points of $\mathbf F_i$ are confined in $\mathcal Z_i$ on a line segment of length $\varrho_{\varepsilon}$, correspond to the initial conditions from which solutions are not defined, and are reached by no integral curve.
%\item if $\varrho_\varepsilon= 0$, then the (unique) singular point of $\mathbf F_i$ lies on $T_i$ and is reached by no integral curve.
%\end{inparaenum}
\item[Case 2.] In region $\mathcal B_i$ one may follow a similar analysis to conclude that the integral curves exit $\mathcal B_i$.
\end{enumerate}
In summary, the vector field \eqref{one obstacle} is safe and globally convergent almost everywhere, i.e., except for a set of initial conditions of measure zero. 
\end{proof}
}
\end{1obstacle}

\subsection{Motion plan in static obstacle environments}
%It is now easy to extend the proposed vector field design on the free space $\mathcal F$ of a workspace $\mathcal W$ populated with $N$ static obstacles.
\begin{Nobstacles}
\textnormal{Assume a workspace $\mathcal W$ of $N$ circular obstacles $\mathcal O_i$, $i\in\{1,\dots,N\}$, positioned such that the inter-obstacle distances $d_{ij}=\|\bm r_{oi}-\bm r_{oj}\|$ satisfy:
\begin{align}
\label{clearance}
d_{ij}\geq \varrho_{\mathcal Zi}+\varrho_{\mathcal Zj}, \; \forall (i,j), \; j\in\{1,\dots,N\}, \; j\neq i.
\end{align}
Then, the vector field $\mathbf F^\star:\R^2\to\R^2$, given as:
\begin{align}
\label{feedback motion plan}
\mathbf F^\star = \prod_{i=1}^N \sigma_i \mathbf F_g + \sum_{i=1}^N (1-\sigma_i) \mathbf F_{oi},
\end{align}
where $\mathbf F_g$ is the normalized attractive vector field \eqref{normalized attractive}, $\mathbf F_{oi}$ is the normalized repulsive vector field \eqref{repulsive normalized full} around an obstacle $\mathcal O_i$, and $\sigma_i$ is the bump function \eqref{bump function} defined in terms of the obstacle function $\beta_i$ given by \eqref{obstacle constraint}, is a safe, almost global feedback motion plan in $\mathcal F$, except for a set of initial conditions of measure zero.}
\end{Nobstacles}

\begin{proof}
\textnormal{By construction, the first term in \eqref{feedback motion plan} cancels the effect of the attractive vector field $\mathbf F_g$ where at least one of the bump functions $\sigma_i=0$, i.e., in the corresponding region $\mathcal Z_i$ around obstacle $\mathcal O_i$. At the same time the second term shapes the corresponding vector field $\mathbf F_{oi}$ in $\mathcal Z_i$. Thus, the attractive vector field $\mathbf F_g$ is activated through \eqref{feedback motion plan} only when $\beta_i<\beta_{\mathcal Zi}$ $\forall i\in\{1,\dots,N\}$, i.e., outside the regions $\mathcal Z_i$. Furthermore, setting the inter-obstacle distance $d_{ij}\geq \varrho_{\mathcal Zi}+\varrho_{\mathcal Zj}$ implies that the repulsive flows around obstacles do not overlap, and therefore are both safe and almost globally convergent to the goal, as proved in Lemma \ref{1obstacle}. This completes the proof.}
\end{proof}

\begin{Remark}
The condition \eqref{clearance} reads that the minimum distance among the boundaries of the obstacles should be at least $2(\varrho+\varrho_{\varepsilon})$. This clearance is not conservative or restrictive in practice, since the parameter $\varrho_\varepsilon$ can be chosen arbitrarily close to zero, or even equal to zero, in case the robot is allowed to touch the obstacle. %In this limiting case, the minimum clearance is equal to $2\varrho$, i.e., equal to the diameter of the robot.
\end{Remark}

\subsection{Control design and simulation results}
Having \eqref{feedback motion plan} at hand, the control design for the unicycle \eqref{unicycle} is now straightforward. We use the control law:
\begin{subequations}
\label{control law}
\begin{align}
u&=k_u\tanh\left(\|\bm r-\bm r_g\|^2\right),\label{linear u}\\
\omega&=-k_\omega(\theta-\varphi)+\dot \varphi,
\end{align}
\end{subequations}
where $\varphi\triangleq\arctan(\frac{\F^\star_y}{\F^\star_x})$ is the orientation of the vector field $\mathbf F^\star$ at a point $(x,y)$, with its time derivative reading:
\begin{align*}
%\dot \varphi &= \frac{d}{dt}\left(\arctan\left(\frac{\F^\star_y}{\F^\star_x}\right)\right)=\frac{1}{1+\left(\frac{\F^\star_y}{\F^\star_x}\right)^2}\;\frac{d}{dt}\left(\frac{\F^\star_y}{\F^\star_x}\right)\\
%&=\frac{{\F^\star_x}^2}{{\F^\star_y}^2+{\F^\star_x}^2}\;\frac{\frac{d\F^\star_y}{dt}\F^\star_x-\F^\star_y\frac{d\F^\star_x}{dt}}{{\F^\star_x}^2}, \mbox{ and since} \; \|\mathbf F^\star\|=1:\\
%\dot \varphi &=\left(\frac{\partial\F^\star_y}{\partial x}\;\dot x+\frac{\partial\F^\star_y}{\partial y}\;\dot y\right)\F^\star_x-\left(\frac{\partial\F^\star_x}{\partial x}\;\dot x+\frac{\partial\F^\star_x}{\partial y}\;\dot y\right)\F^\star_y\\
\dot \varphi \overset{\eqref{unicycle}}{=}\left(\left(\begin{matrix}\frac{\partial\F^\star_y}{\partial x}\;c\theta+\frac{\partial\F^\star_y}{\partial y}\;s\theta\end{matrix}\right)\F^\star_x-\left(\begin{matrix}\frac{\partial\F^\star_x}{\partial x}\;c\theta+\frac{\partial\F^\star_x}{\partial y}\;s\theta\end{matrix}\right)\F^\star_y\right)u,
\end{align*}
with the linear velocity $u$ given by \eqref{linear u}, see in \cite{Panagou_Technical_Report_Vector_Fields}, and $k_\omega>0$, %picked sufficiently large compared to 
$k_u>0$. Then, the orientation $\theta$ of the unicycle is \ac{ges} to the safe orientation $\varphi$, and %; in this way, the trajectories $\theta(t)$ evolve as the fast subsystem in a singular-perturbations sense. 
the robot flows along the integral curves of $\mathbf F^\star$ until converging to $\bm r_g$.

To demonstrate the efficacy of the proposed navigation and control design we consider the motion of a robot in an environment with $N=10$ static obstacles (Fig. \ref{fig:path1}), where the goal position is $\bm r_g=\left[-0.1 \;\; 0.08\right]^T$. The radii of the obstacles are set equal to $\varrho_{oi}=0.03$. The blending zone $\mathcal D_i$ around each obstacle $\mathcal O_i$ is illustrated between the boundary surfaces $S_i$ (black line) and $T_i$ (red line), respectively. The resulting collision-free path under the control law \eqref{control law}, with the control gains picked equal to $k_u=0.075$, $k_r=2.5$, are depicted in blue color.
\begin{figure}
\centering
\includegraphics[width=0.95\columnwidth,clip]{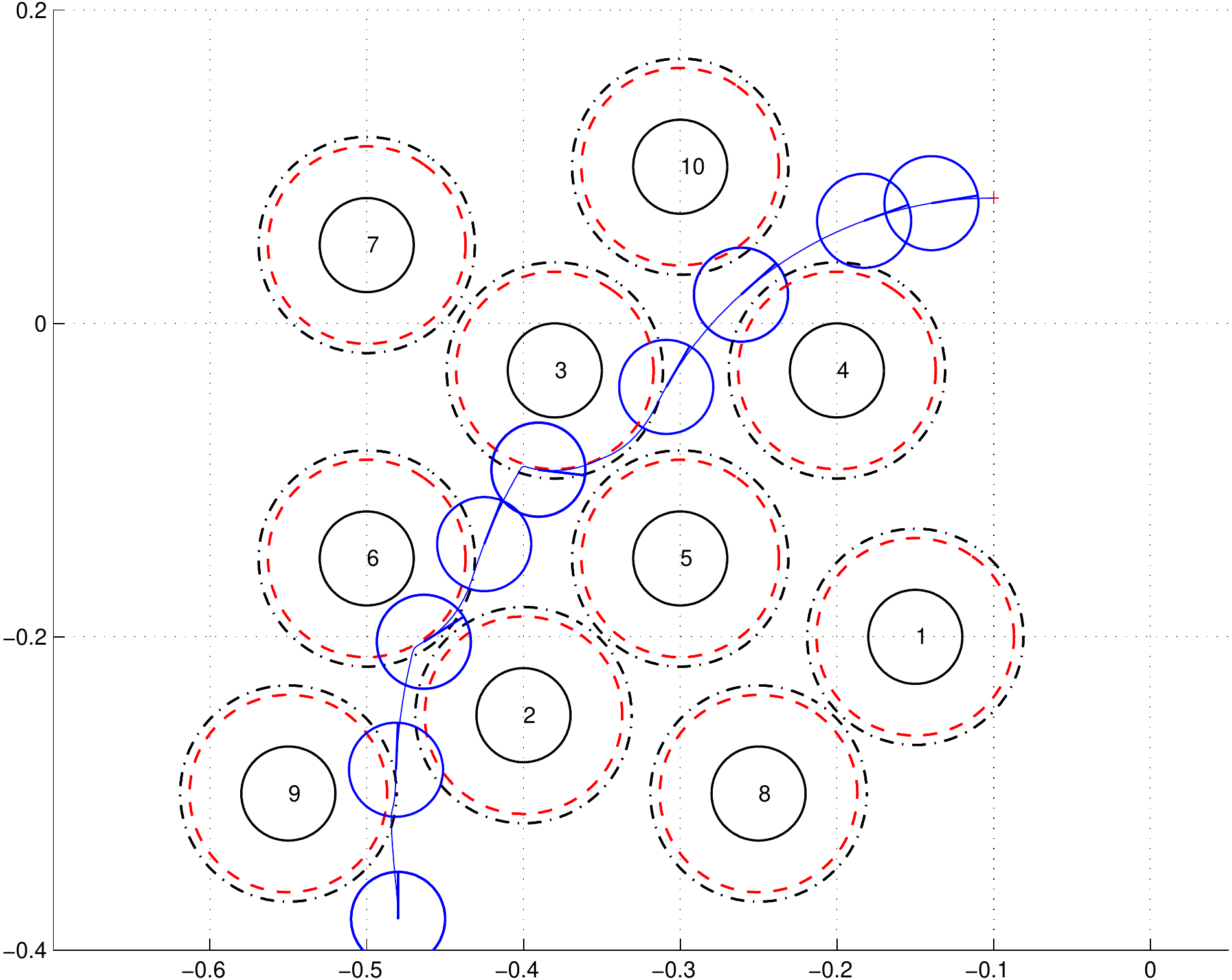}
\caption{The path of a unicycle in a obstacle environment.}
\label{fig:path1}
\end{figure}
\vspace{1mm}
\begin{Remark}
The integral curves of $\mathbf F_{oi}$ in the region $\mathcal A_i$ around an obstacle $\mathcal O_i$ forces the robot to perform a sharp maneuver in order to follow the tangential direction and avoid collision. This in practice is plausible for unicycle-type vehicles (e.g. differentially driven mobile robots), yet it may not be desirable for input-constrained vehicles, such as car-like vehicles and aircraft. Our current work focuses in encoding curvature constraints via \eqref{2-d dipolar}.
\end{Remark}

\section{Extension to dynamic environments}\label{Simulations}
Consider $N$ agents $i\in\{1,2,\dots,N\}$ of unicycle kinematics which are assigned with the task to converge to goal configurations $\bm q_{gi}$ while avoiding collisions. 

Each agent $i$ has a circular communication/sensing region $\mathcal C_i$ of radius $R_c$ centered at $\bm r_i=\begin{bmatrix}x_i&y_i\end{bmatrix}^T$, denoted as: $$\mathcal C_i : \{\bm r\in\R^2 \; | \; \|\bm r_i - \bm r\|\leq R_c\},$$
and can reliably exchange information with any agent $j\neq i$ which lies within its communication region $\mathcal C_i$. In other words, we say that a pair of agents $(i,j)$ is connected, or equivalently, that agent $j$ is neighbor to agent $i$, as long as the inter-agent distance $d_{ij}=\|\bm r_i - \bm r_j\|\leq R_c$. Denote the set of neighbors $j\neq i$ of agent $i$ with $\mathcal N_i$. %The motion of each agent $i\in\{1,\dots,N\}$ is described as:
%\begin{align}
%\label{agent model}
%\dot{\bm q}_i=\begin{bmatrix}\cos \theta_i & \sin \theta_i & 0 \end{bmatrix}^T u_i + \begin{bmatrix}0 & 0 & 1 \end{bmatrix}^T  \omega_i,
%\end{align}
%where $\bm q_i=\left[x_i\;\;y_i\;\;\theta_i\right]^T$ is the configuration vector of agent $i$ \ac{wrt} a global frame $\mathcal G$ and $u_i$, $\omega_i$ are the linear and angular velocities of agent $i$, respectively. Each agent $i$ knows its own goal configuration $\bm q_{gi}$, but not the goal configurations $\bm q_{gj}$ of the other agents $j\neq i$, $j\in\{1,\dots,N\}$.

Agents $j\neq i$ serve as dynamic (moving) obstacles to agent $i$. Navigating safely to an assigned goal $\bm q_{gi}$ is then reduced into finding a feedback motion plan $\mathbf F^\star_i$ such that its integral curves:
\begin{inparaenum}
\item[(i)] point into the interior of the collision-free space $\mathcal F_i$ on the boundaries of the agents $j\neq i$, and
\item[(ii)] converge to the goal $\bm q_{gi}$.
\end{inparaenum}
Towards this end, we would like to employ a vector field $\mathbf F^\star_i$ for each agent $i$ as:
\begin{align}
\label{feedback motion plan dynamic}
\mathbf F^\star_i = \prod_{j\in\mathcal N_i} \sigma_{ij} \mathbf F_{gi} + \sum_{j\in\mathcal N_i} (1-\sigma_{ij}) \mathbf F^i_{oj},
\end{align}
where the attractive term $\mathbf F_{gi}$ is taken out of \eqref{normalized attractive}, the bump function $\sigma_{ij}$ is defined later on, and the repulsive term $\mathbf F^i_{oj}$ around each each agent $j\neq i$ \emph{is replaced with a normalized repelling node},\footnote{The tangential repulsive vector field \eqref{repulsive full} defined for static obstacles is not a suitable choice for the dynamic case; the reason is that the repulsive integral curves of the vector field \eqref{feedback motion plan} of agent $i$ around agent $j$ are rendered an invariant set under the proposed velocity coordination protocol, forcing thus the trajectories $\bm r_i(t)$, $\bm r_j(t)$ of a pair of agents $i$, $j$ converge to undesired locations away from the goal locations $\bm r_{gi}$, $\bm r_{gj}$, see also the analysis in \cite{Panagou_ICRA15}.} given out of:
\begin{subequations}
\label{normalized repelling node}
\begin{align}
\F^i_{xoj}= \frac{x_i-x_j}{\sqrt{(x_i-x_j)^2+(y_i-y_j)^2}},\\
\F^i_{yoj}= \frac{y_i-y_j}{\sqrt{(x_i-x_j)^2+(y_i-y_j)^2}}.
\end{align}
\end{subequations}
In order to utilize the (almost global) convergence and safety guarantees applying to the static case, we need to ensure that the repulsive flows around agents do not overlap, for any pair of agents $(i,j)$. Recall from the static case that this condition equivalently reads as that the minimum distance $d_m$ between any pair of (moving, in the dynamic case) agents $(i,j)$ is $d_m=2(2\varrho+\varrho_\epsilon)$, or equivalently, that the minimum clearance between any pair of agents is $(2\varrho+\varrho_\epsilon)$, where $\varrho_\epsilon>0$ arbitrarily small and $\varrho$ is the radius of the agents.

In that respect, the bump function $\sigma_{ij}$ in \eqref{feedback motion plan dynamic} is defined as:
\begin{align}
\label{bump function dynamic}
\sigma_{ij} = \left\{
              \begin{array}{ll}
                1, & \; \hbox{for }\;d_m \leq d_{ij}\leq d_r; \\
                a\;{d_{ij}}^3 + b\;{d_{ij}}^2 + c\; d_{ij} + d, & \; \hbox{for }\;d_r< d_{ij}< d_c; \\
                0, & \; \hbox{for }\;d_{ij}\geq d_c;
              \end{array}
            \right.
\end{align}
where the coefficients $a, b, c, d$ have been computed as:
\begin{align*}
a &= -\frac{2}{(d_r-d_c)^3}, & b &= \frac{3(d_r+d_c)}{(d_r-d_c)^3},\\
c &= -\frac{6\;d_r d_c}{(d_r-d_c)^3}, & d &= \frac{{d_c}^2(3d_c-d_r)}{(d_r-d_c)^3},
\end{align*}
so that \eqref{bump function} is a $\mathcal C^2$ function, $d_m\geq 2(2\varrho+\varrho_\epsilon)$, $d_m<d_r<d_c$.
\begin{Remark}
The communication/sensing range of each agent $i$ should be $R_c\geq d_c$.
\end{Remark}

\subsection{Control design}
Each agent $i$ moves under the control law:
\begin{subequations}
\label{control law agent i}
\begin{align}
u_i &= \left\{
         \begin{array}{rc}
           \max\left\{0,\min\limits_{j\in \mathcal N_i | J_j<0} u_{i|j}\right\}, & \hbox{$d_m\leq d_{ij}\leq R_c$,}\\
            u_{ic}, & \hbox{$R_c<d_{ij}$;} \\
         \end{array}
       \right.,\label{ui}\\
\omega_i &= -k_{\omega i}\left(\theta_i-\varphi_i\right)+\dot \varphi_i,\label{wi}
\end{align}
\end{subequations}
where: $\varphi_i$ is the orientation of the vector field $\mathbf F^\star_i$ at a point $(x,y)$, the vector field $\mathbf F^\star_i$ is given by \eqref{feedback motion plan dynamic}, $u_{i|j}$ is the safe velocity of agent $i$ \ac{wrt} an agent $j$ lying in the communication region of $\mathcal C_i$ of agent $i$, given as:
\begin{align}
\label{safe velocity}
u_{i|j}&=u_{ic}\;\frac{d_{ij}-d_m}{R_c-d_m}+\varepsilon_i \; u_{is|j} \; \frac{R_c-d_{ij}}{R_c-d_m},
\end{align}
with the terms in \eqref{safe velocity} defined as:
\begin{align*}
u_{ic}&=k_{ui} \tanh(\|\bm r_i - \bm r_{gi}\|),\quad u_{is|j}= u_j \;\frac{{\bm r_{ji}}^T\bm \eta_j}{{\bm r_{ji}}^T\bm \eta_i},\\
\bm \eta_i &=\left[\begin{matrix}\cos\varphi_i\\\sin\varphi_i\end{matrix}\right],\quad J_j={\bm r_{ji}}^T\bm \eta_i, \quad \bm r_{ji}=\bm r_i-\bm r_j,
\end{align*}
and $\varepsilon_i>1$, $k_{ui},k_{\omega i}>0$.

\begin{Multi-Robot}\label{Multi-Robot}
\textnormal{Consider $N$ agents $i\in\{1,\dots,N\}$ assigned to converge to goal configurations $\bm q_{gi}$. Then, under the control law \eqref{control law agent i} each agent safely converges to its goal configuration almost globally, except for a set of initial conditions of measure zero.}

\begin{proof}\textnormal{
The closed loop trajectories of each agent $i$ are forced to flow along the vector field \eqref{feedback motion plan dynamic}. If $d_{ij}(t)>R_c$, $\forall t\geq0$ and $\forall j\in\{1,\dots,N\}$, then $\sigma_{ij}(t)=1$, implying that agent $i$ flows safely along \eqref{normalized attractive} and converges to $\bm q_{gi}$.
\\
Let us now assume that at some time $t\geq 0$ the distance $d_{ij}(t)$ between a pair of agents $(i,j)$ is $d_{ij}(t)\leq R_c$. By definition agent $i$ lies in the sensing/communication region of agent $j$ and vice versa, which implies that they exchange information on their current positions $\bm r_i(t)$, $\bm r_j(t)$ and velocities $\bm \nu_i(t)$, $\bm \nu_j(t)$. Consider the blending region $\mathcal D_i:\{\bm r_j\in\R^2\;|\; d_r < \|\bm r_i-\bm r_j\| < d_c\}$ and the surfaces:
\begin{align*}
S_{i}(t) &: \{\bm r_i(t), \bm r_j(t)\in \R^2  \; | \; \|\bm r_i(t)-\bm r_j(t)\|-d_c=0\},\\
T_i(t) &: \{\bm r_i(t), \bm r_j(t)\in \R^2  \; | \; \|\bm r_i(t)-\bm r_j(t)\|-d_r=0\}.
\end{align*}
%\begin{nosingularities}
%The integral curves of the vector field $\mathbf F_i^\star$ are always well-defined, since no singularities occur in the blending region $\mathcal D_i$.
%\end{nosingularities} 
%\begin{proof}
%It is straightforward to verify the argument by using the same reasoning as in Lemma \ref{1obstacle}.
%\end{proof} 
\begin{nocollisions}
Agent $i$ avoids collision with any of its neighbor agents $j\in\mathcal N_i$.
\end{nocollisions}
\begin{proof}
Collision-free motion is realized as ensuring that $d_{ij}(t)\geq 2\varrho$, $\forall t\geq 0$, for any pair $(i,j)$. Let us consider the time derivative of inter-agent distance function, which after some calculations reads:
\begin{align}
\label{dijdt}
\frac{d}{dt}\;d_{ij}&=\frac{(x_i-x_j)(\dot x_i-\dot x_j)}{d_{ij}}+\frac{(y_i-y_j)(\dot y_i-\dot y_j)}{d_{ij}}\nonumber\\
&\overset{\eqref{unicycle}}{=}\frac{u_i \; {\bm r_{ji}}^T\bm \eta_i - u_j \; {\bm r_{ji}}^T\bm \eta_j}{d_{ij}}.
\end{align} 
The control law \eqref{control law agent i} renders the value of the time derivative \eqref{dijdt} positive when the value of the distance function is $d_{ij}=d_m$, implying thus that the inter-agent distance is forced to increase. Since $d_m>2\varrho$, this further implies that collisions are avoided.
\end{proof}
In order to draw conclusions about the convergence of the agents' trajectories to their goal configurations we need to examine the behavior of the integral curves around the switching surfaces $S_i(t)$, $T_i(t)$. With the vector fields $\mathbf F_i^\star$, $\mathbf F_j^\star$ well-defined everywhere in the corresponding blending regions, and the linear velocities $u_i$, $u_j$ vanishing only at the goal locations $\bm r_{gi}$, $\bm r_{gj}$, we are interested in identifying conditions under which the system trajectories $\bm r_i(t)$, $\bm r_j(t)$ are forced to get stuck on $S_{i}(t)$, or on $T_i(t)$, for infinite amount of time. This can be seen as identifying sufficient conditions of the appearance of (chattering) Zeno behavior, or Zeno points \cite{Ames_ACC05}. A sufficient condition on the appearance of Zeno points is given in \cite{Ceragioli_NonlinearAnalysis2006}, Theorem 2. Based on this result, we study under which conditions the system (i.e., agents') trajectories converge to a Zeno point. Consider the case with $N=2$ agents. Denote the dynamics of the $k$-th agent as $\bm{\dot q}_k=\bm f_k(\bm q_k)$, $k\in\{i,j\}$, $\bm q=\left[{\bm q_i}^T \;\; {\bm q_j}^T\right]^T$, $\bm r=\left[{\bm r_i}^T \;\; {\bm r_j}^T\right]^T$, and take:
\begin{align}
\nabla S_{i} \bm f(\bm q) = 2 u_i\; {\bm r_{ij}}^T \left[\begin{matrix}\cos\theta_i\\ \sin\theta_i\end{matrix}\right]-2 u_j \; {\bm r_{ij}}^T \left[\begin{matrix}\cos\theta_j\\ \sin\theta_j\end{matrix}\right],
\end{align}
where $\bm r_{ij}=\bm r_i-\bm r_j$. Note that the control law \eqref{wi} renders the orientation $\theta_k$ of the $k$-th agent \ac{ges} to the orientation $\varphi_k$ of the vector field $\mathbf F_k^\star$. Thus, the unit vector $\left[\cos\theta_k\;\;\;\sin\theta_k\right]^T$ coincides with the vector field $\mathbf F_k^\star(\bm r_k)$, evaluated at $\bm r_k\in\R^2$. With this at hand %, one has:
%\begin{align*}
%\nabla S_{i}^- &\; \bm f(\bm q)=2u_i\; {\bm r_{ij}}^T \mathbf F_{gi}-2 u_j \; {\bm r_{ij}}^T \mathbf F_{gj},\\
%\nabla S_{i}^+ &\; \bm f(\bm q)=2u_i\; \sigma_{ij}\; {\bm r_{ij}}^T \mathbf F_{gi}-2 u_j\; \sigma_{ji} \; {\bm r_{ij}}^T \mathbf F_{gj} \\
%&+ 2 u_i(1-\sigma_{ij})\; {\bm r_{ij}}^T \mathbf F^i_{oj} - 2 u_j (1-\sigma_{ji})\; {\bm r_{ij}}^T \mathbf F^j_{oi},
%\end{align*}
and after some algebraic calculations one has:
\begin{align*}
\nabla S_{i}^- &\; \bm f(\bm q)=\underbrace{2u_i\; {\bm r_{ij}}^T \mathbf F_{gi}-2 u_j \; {\bm r_{ij}}^T \mathbf F_{gj}}_A,\\
\nabla S_{i}^+ &\; \bm f(\bm q)= \sigma_{ij}\;\left(2u_i\; {\bm r_{ij}}^T \mathbf F_{gi}-2 u_j\; {\bm r_{ij}}^T \mathbf F_{gj}\right) \\
&+ (1-\sigma_{ij})\;\underbrace{\left(2 u_i\; {\bm r_{ij}}^T \mathbf F^i_{oj} - 2 u_j \; {\bm r_{ij}}^T \mathbf F^j_{oi}\right)}_B\;.
\end{align*}
The set of Zeno points is: $$Z_{i}=\{\bm r \in \R^{2N} \;| \; \nabla S_{i}^- \;\bm f(\bm q)=\nabla S_{i}^+ \; \bm f(\bm q)=0\},$$ which reads: 
\begin{align*}
A=\sigma_{ij}A+(1-\sigma_{ij})B&=0\Rightarrow A=B=0\Rightarrow \\
u_i\;{\bm r_{ij}}^T(\mathbf F_{gi}-\mathbf F^i_{oj})&=u_j\;{\bm r_{ij}}^T(\mathbf F_{gj}-\mathbf F^j_{oi}).
\end{align*} 
Not surprisingly, the set $Z_i$ is depended on the current positions $\bm r_i$, $\bm r_j$ and the goal locations $\bm q_{gi}$, $\bm q_{gj}$. The Zeno condition reduces to $(\mathbf F_{gi}+\mathbf F_{gj})=\bm 0$, which corresponds to current positions $\bm r_i(t)$, $\bm r_j(t)$ and goal locations $\bm r_{gi}$, $\bm r_{gj}$ lying on the same line. Then, the set of initial conditions (positions) from which agents' trajectories converge to the set $Z_{i}$ is confined on $\mathbb R$, i.e., on a lower dimensional manifold, and as thus is of measure zero. The same analysis holds along the switching surface $T_{i}(t)$, yielding exactly the same condition as before regarding on the appearance of Zeno points.
\vspace{1mm}
\\
The case of $N>2$ agents can be treated accordingly. Consider an agent $i$ lying at distance $d_{im}\leq d_c$ \ac{wrt} $M\leq (N-1)$ agents $m\neq i$. The vector field $\mathbf F_i^\star$ includes the repulsive effect $\mathbf F^i_{om}$ of all $M$ connected agents. To check whether undesired singularities appear, one needs to consider the norm $\|\mathbf F^i\|$ in the blending region $\mathcal D_i$. The analytical expression is more involved compared to the $N=2$ case.
%\footnote{Let us illustrate the case of $N=3$ robots $i, \; j, \; k$, and compute: $\|\mathbf F_i^\star\|={\sigma _{ij}}^2{\sigma _{ik}}^2+\left(1-\sigma_{ij}\right)^2+\left(1-\sigma_{ik}\right)^2+
%2\;\left(1-\sigma_{ij}\right)\left(1-\sigma_{ik}\right)\cos\gamma+
%2\;\sigma_{ij}\sigma_{ik}\left(1-\sigma_{ij}\right)\cos\alpha+
%2\;\sigma_{ij}\sigma_{ik}\left(1-\sigma_{ik}\right)\cos\beta$, where $\alpha$ is the angle between $\mathbf F_{gi}$, $\mathbf F^i_{oj}$, $\beta$ is the angle between $\mathbf F_{gi}$, $\mathbf F^i_{ok}$, and $\gamma$ is the angle between $\mathbf F^i_{oj}$, $\mathbf F^i_{ok}$. As expected, the norm $\|\mathbf F_i^\star\|$ depends on the positions $\bm r_i$, $\bm r_j$, $\bm r_k$ of the agents and their corresponding goal locations $\bm r_{gi}$, $\bm r_{gj}$, $\bm r_{gk}$. To further simplify the expression, denote $\sigma_{ij}=\sigma_i$, $\sigma_{ik}=\mu_i\sigma_{ij}=\mu_i\sigma_{i}$, $\mu_i\in(0,1)$ and write:
%\begin{align}
%\label{polynomial}
%\|\mathbf F_i^\star\|&=p(\sigma_i)={\mu_i}^2{\sigma _{i}}^4-2\mu_i\left(\cos\alpha+\mu_i\cos\beta\right){\sigma_i}^3\nonumber\\
%&+\left(1+{\mu_i}^2+2\mu_i(\cos\alpha+\cos\beta+\cos\gamma)\right){\sigma_i}^2\nonumber\\
%&-2(1+\mu_i)(1+\cos\gamma)\sigma_i+2(1+\cos\gamma).
%\end{align}
%The norm is non-vanishing as long as the quartic polynomial \eqref{polynomial} has complex roots only, or real roots out of the interval $(0,1)$. The conditions can be identified by studying the discriminant of the polynomial.} 
To make the probability of more than one agents lying in the blending region $\mathcal D_i$ as low as possible, one can define the width $d_c-d_r\rightarrow 0$. 
Define also: $$S_{im}(t) : \{\bm r_i(t), \bm r_m(t)\in \R^2  \; | \; \|\bm r_i(t)-\bm r_m(t)\|-d_c=0\}$$ the $M\leq (N-1)$ switching surfaces of agent $i$ \ac{wrt} its neighbors $m$. The conditions on the appearance of Zeno points around each switching surface read: $$\nabla S_{im}^-\bm f(\bm r_i,\bm r_m)=\nabla S_{im}^+\bm f(\bm r_i,\bm r_m)=0, \forall m\in\{1,\dots,N\}.$$ 
This results in $\frac{N M}{2}$ switching surfaces, since for any pair of agents $(i,m)$ it holds that: $S_{im}=S_{mi}$, and $N M$ Zeno conditions. 
\\
Now, note that the $\frac{N M}{2}$ Zeno conditions $S^-_{im}\bm f(\bm r_i,\bm r_m)=0$ introduce $2\;\frac{N M}{2}=N M$ unknown terms of the form ${\bm r_{im}}^T\mathbf F_{gi}$, ${\bm r_{im}}^T\mathbf F_{gm}$. 
\\
In the same spirit, the $\frac{N M}{2}$ Zeno conditions $S^+_{im}\bm f(\bm r_i,\bm r_m)=0$ additionally introduce $\frac{N M}{2} (M-1) 2 = N M (M-1)$ unknown terms of the form ${\bm r_{im}}^T\mathbf F^i_{ok}$, ${\bm r_{im}}^T\mathbf F^m_{ok}$. 
\\
In total, one has $N M^2$ unknown terms and $N M$ equations. Given that the $N$ goal locations $\bm r_{gi}, \bm r_{gm}, \dots, $ are known, the number of unknown terms reduces to $NM^2-NM=NM(M-1)$. To have as many equations as unknown terms, it should hold that $M-1=1 \; \Rightarrow \; M=2$. This implies that each agent $i\in\{1,\dots,N\}$ is connected with at most $M=2$ agents; note that this is irrespective of the total number of agents $N$. %For $N=3$ in particular, we have $NM=6$ unknown terms. 
Then, the geometric conditions which result in Zeno points are given as the solutions of the resulting linear system; these solutions express $NM$ relations of the form $\bm r_{im}^T\mathbf F^i_{ok}$, $\bm r_{im}^T\mathbf F^m_{ok}$, which dictate the Zeno points, i.e., the Zeno positions among the $N$ agents. Then, the set of initial conditions from which the agents converge to these Zeno positions are confined to a lower dimensional manifold, since they correspond to initial positions confined on a line, and to a specific initial orientation for each agent, and as thus are of measure zero. 
\\
Finally, let us note that the case of $M>2$ neighbors is not of interest for the proposed algorithm, as each agent $i$ makes the avoidance decision \ac{wrt} the worst-case neighbor agent, i.e., \ac{wrt} the agent which is more susceptible to collision. This is realized via considering the safe velocity $u_{i|m}$ \ac{wrt} each neighbor agent $m$ and taking the minimum over safe velocities in the definition of the linear velocity control law \eqref{ui}. The maximum function is defined to ensure that each agent $i$ will never be forced to move with negative linear velocity, i.e., backwards; this is to ensure that there is no possibility of back-to-back colliding agents. 
\\
In summary, the motion of each agent $i$ remains collision-free \ac{wrt} its neighbor agents $j\in\mathcal N_i$ under the control law \eqref{control law agent i}, and each agent $i$ converges to its goal location $\bm q_{gi}$ almost globally, except for a set of initial configurations of measure zero. This completes the proof.}
\end{proof}
\end{Multi-Robot}
\begin{Remark}
Theorem \ref{Multi-Robot} justifies that the set of initial conditions for which the multi-robot system exhibits Zeno trajectories (chattering across a switching surface for infinite amount of time) which result in robots getting stuck away from their goals, is of measure zero. To avoid sliding along a switching surface, which can be seen as ``finite-time chattering", one can employ hysteresis logics \cite{Liberzon}.
\end{Remark}

\subsection{Simulation Results}

We consider $\N=30$ agents which are moving towards their goal locations (depicted with square markers) starting from goal positions (depicted with cross markers) while avoiding collisions, see the resulting paths in Fig. \ref{fig:multi_scenario1, fig:multi_scenario2}. The goal locations are defined sufficiently far apart so that the communication regions do not overlap when agents lie on their goal locations.

\begin{figure*}
{
\centering
\subfigure{\includegraphics[width=0.32\textwidth,clip]{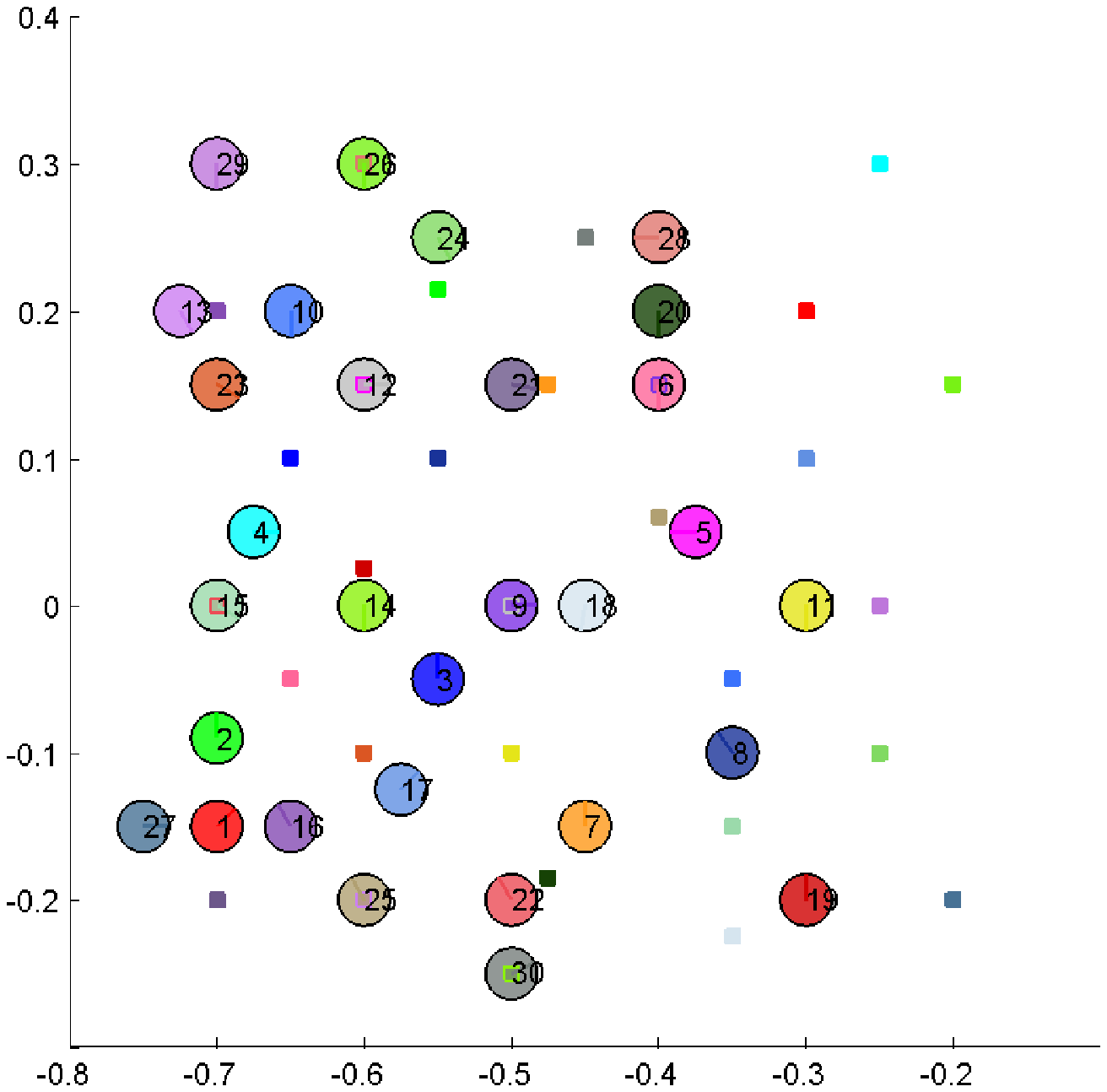}}
\subfigure{\includegraphics[width=0.32\textwidth,clip]{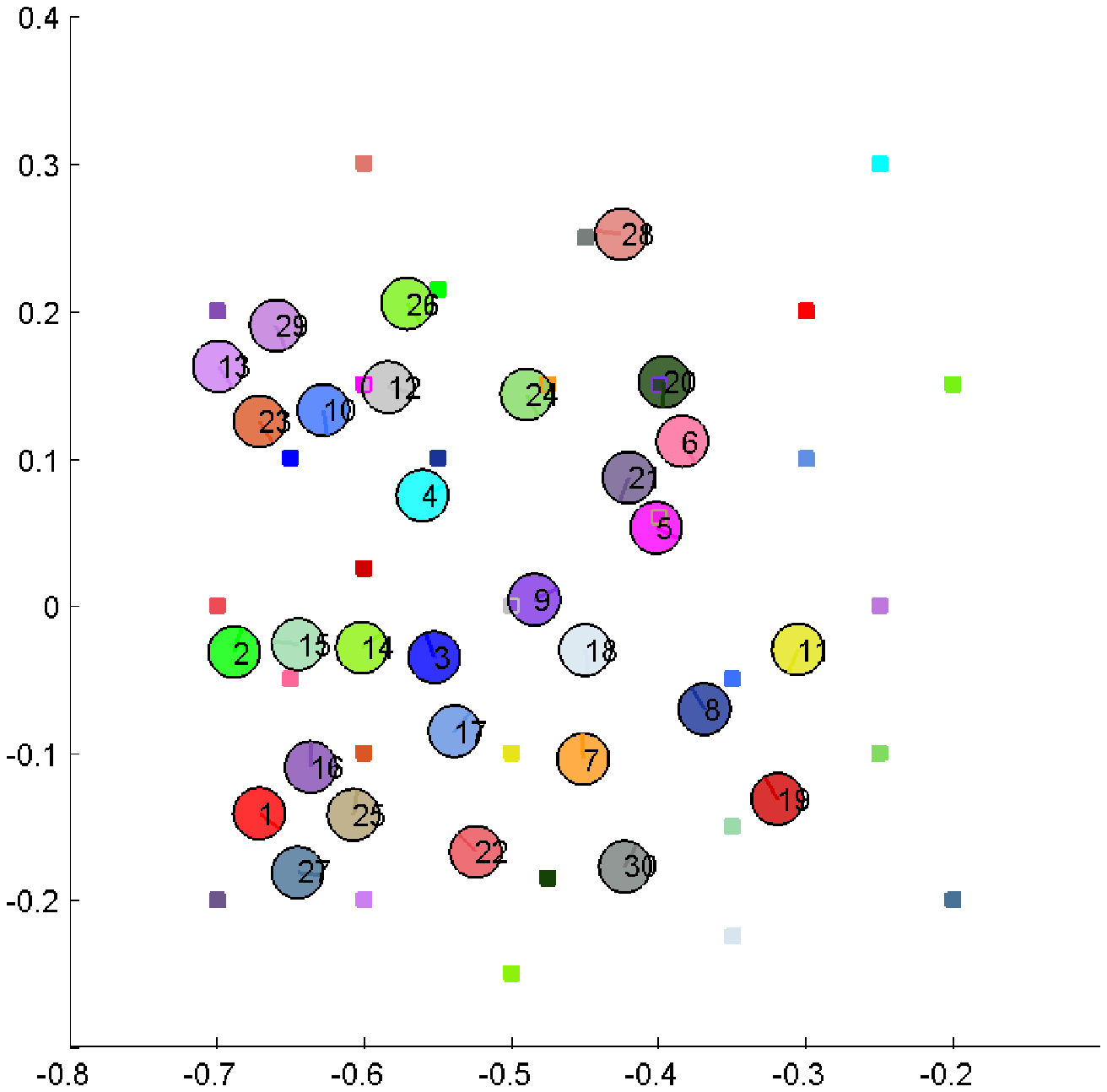}}
\subfigure{\includegraphics[width=0.32\textwidth,clip]{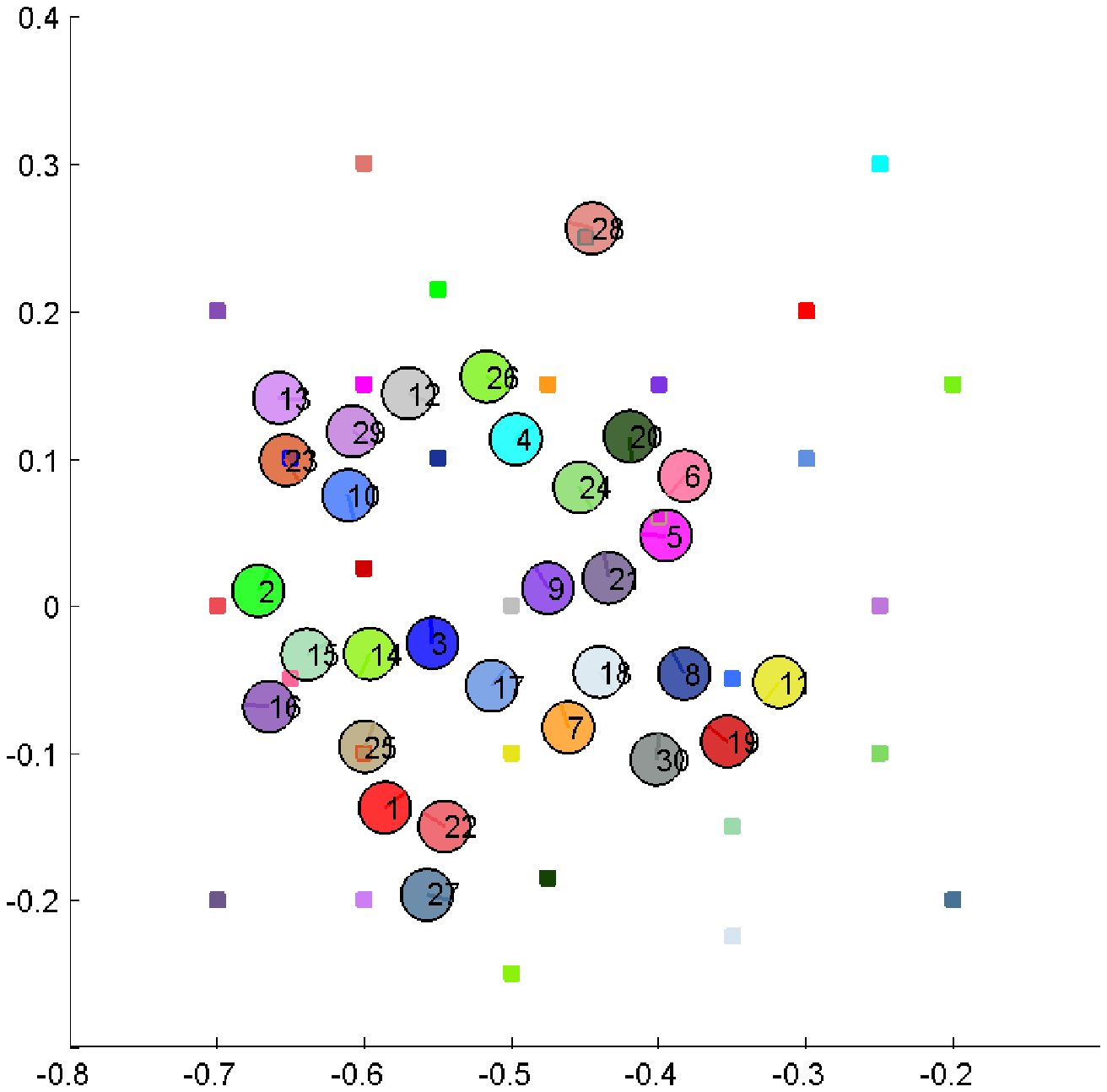}}
\subfigure{\includegraphics[width=0.32\textwidth,clip]{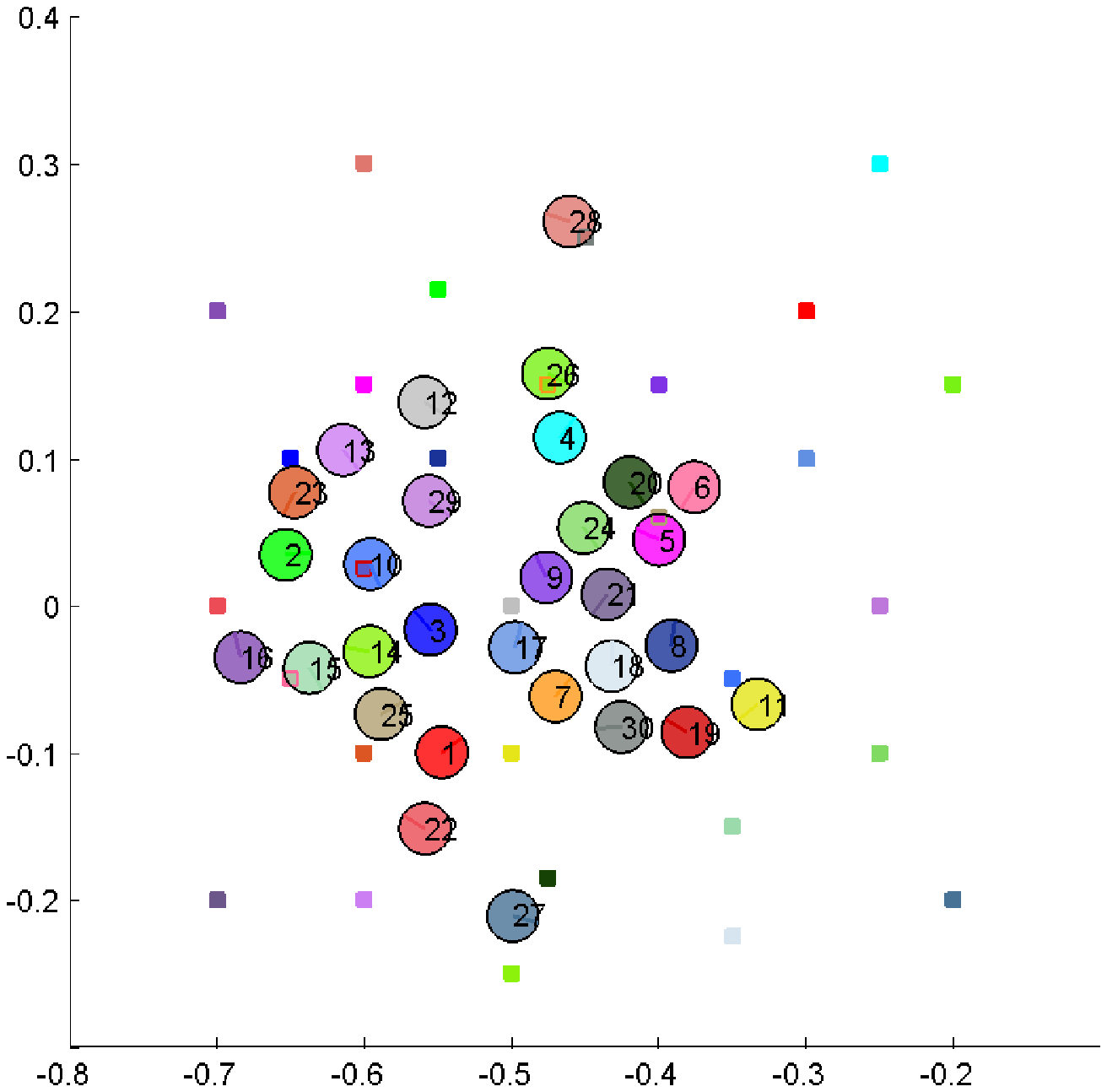}}
%\subfigure{\includegraphics[width=0.3\textwidth,clip]{Journal_SimT/jshot5.eps}}
\subfigure{\includegraphics[width=0.32\textwidth,clip]{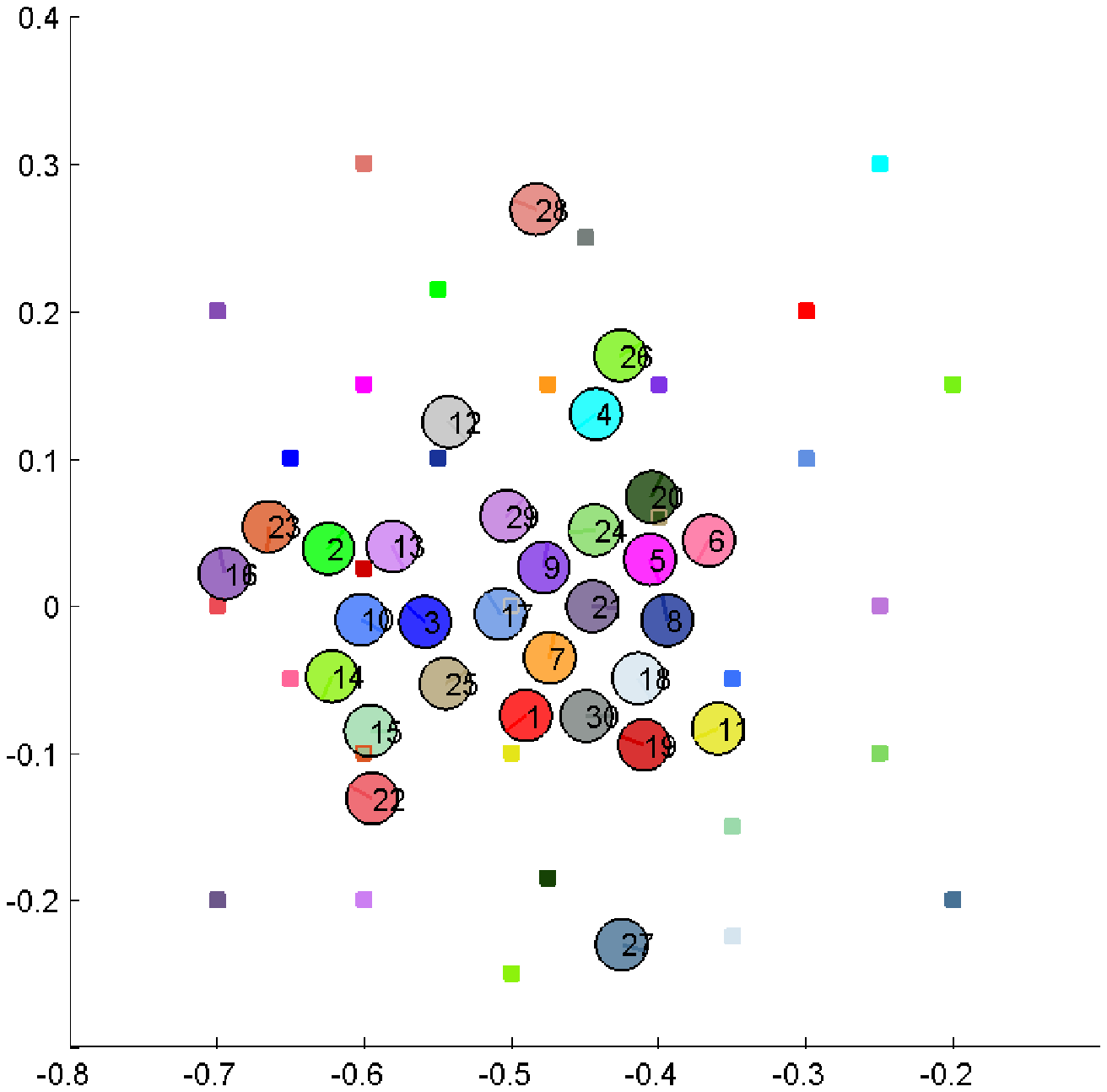}}
\subfigure{\includegraphics[width=0.32\textwidth,clip]{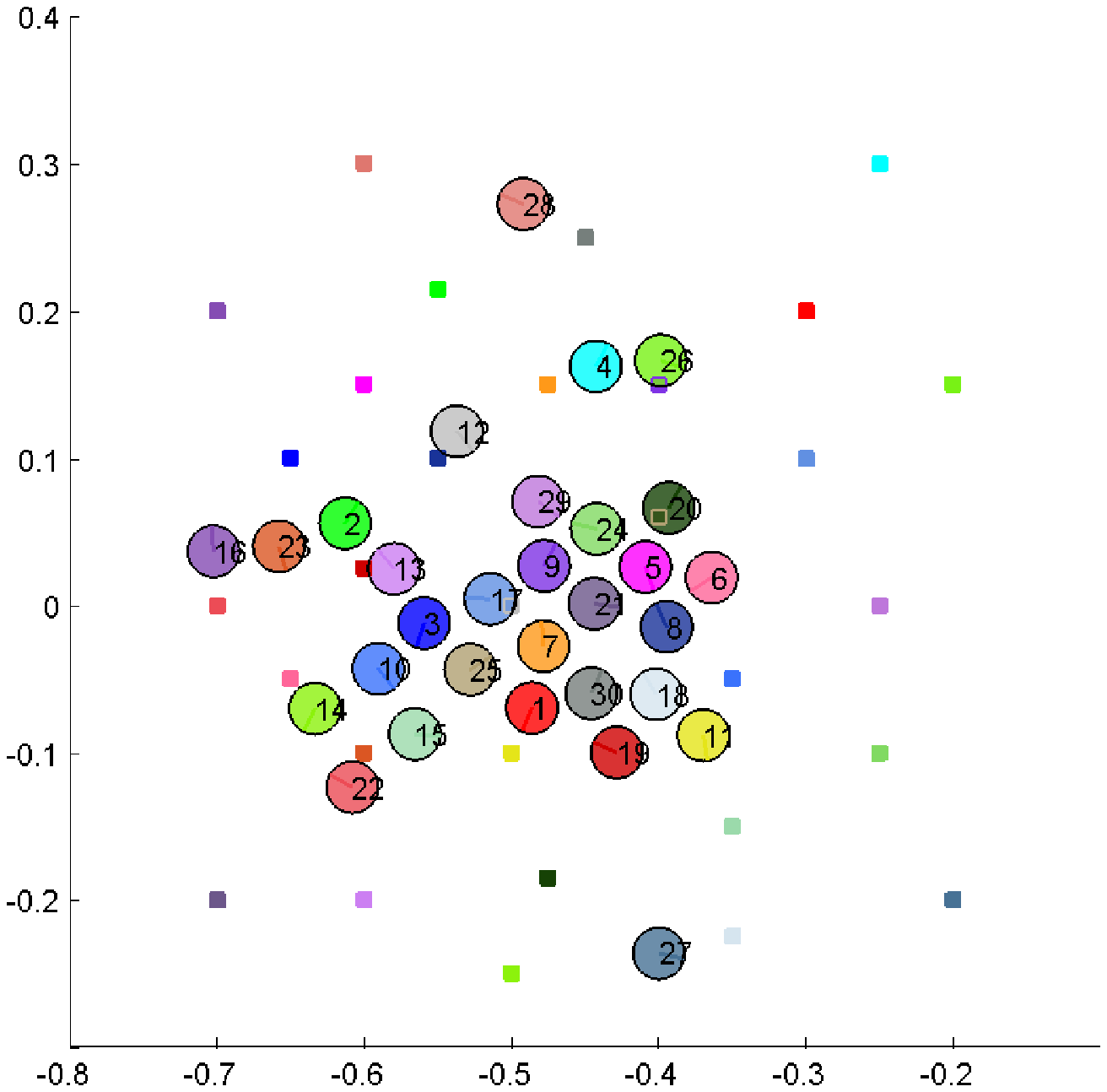}}
%\subfigure{\includegraphics[width=0.3\textwidth,clip]{Journal_SimT/jshot8.eps}}
%\subfigure{\includegraphics[width=0.3\textwidth,clip]{Journal_SimT/jshot9.eps}}
\subfigure{\includegraphics[width=0.32\textwidth,clip]{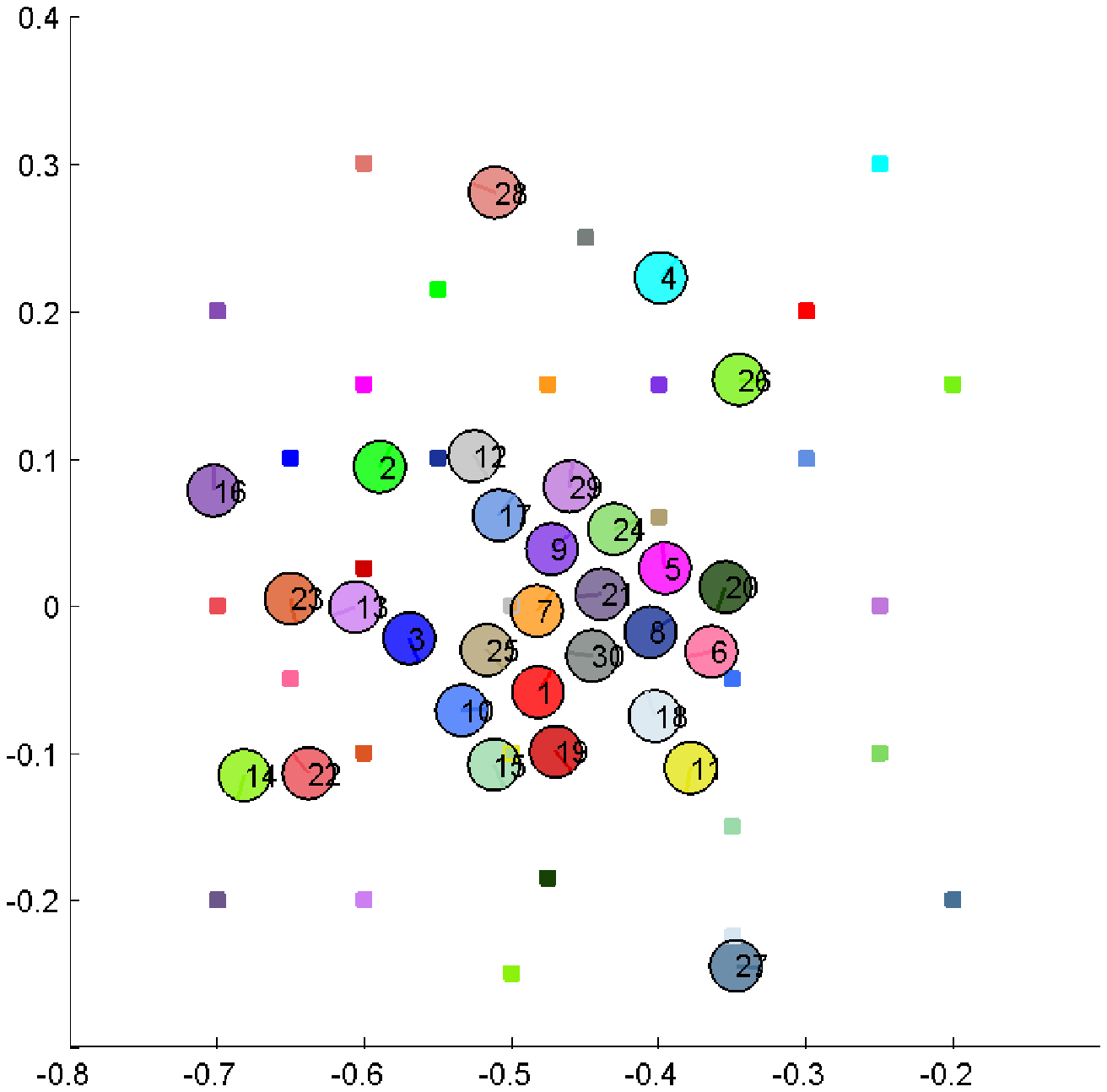}}
%\subfigure{\includegraphics[width=0.3\textwidth,clip]{Journal_SimT/jshot11.eps}}
%\subfigure{\includegraphics[width=0.3\textwidth,clip]{Journal_SimT/jshot12.eps}}
\subfigure{\includegraphics[width=0.32\textwidth,clip]{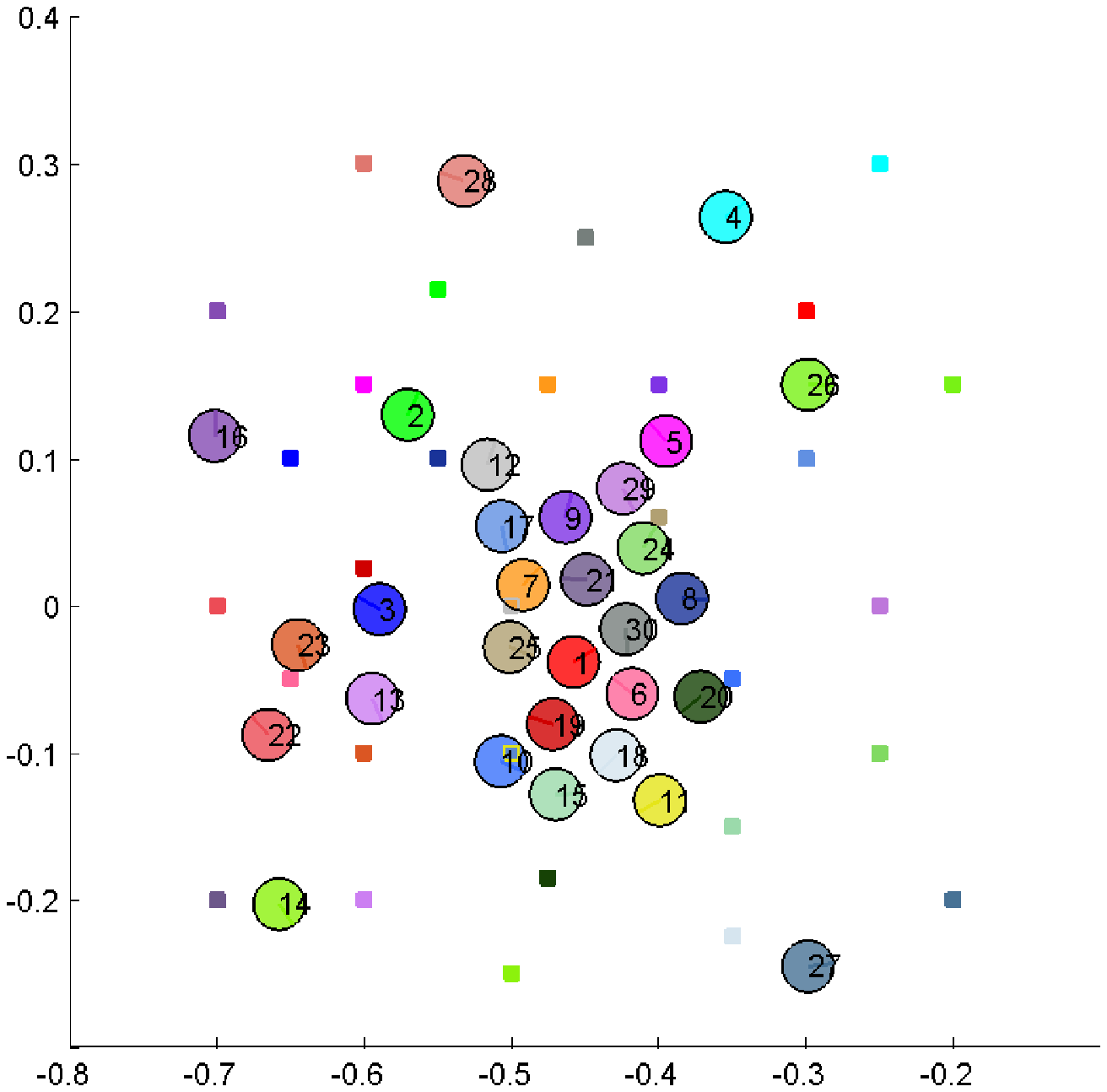}}
\subfigure{\includegraphics[width=0.32\textwidth,clip]{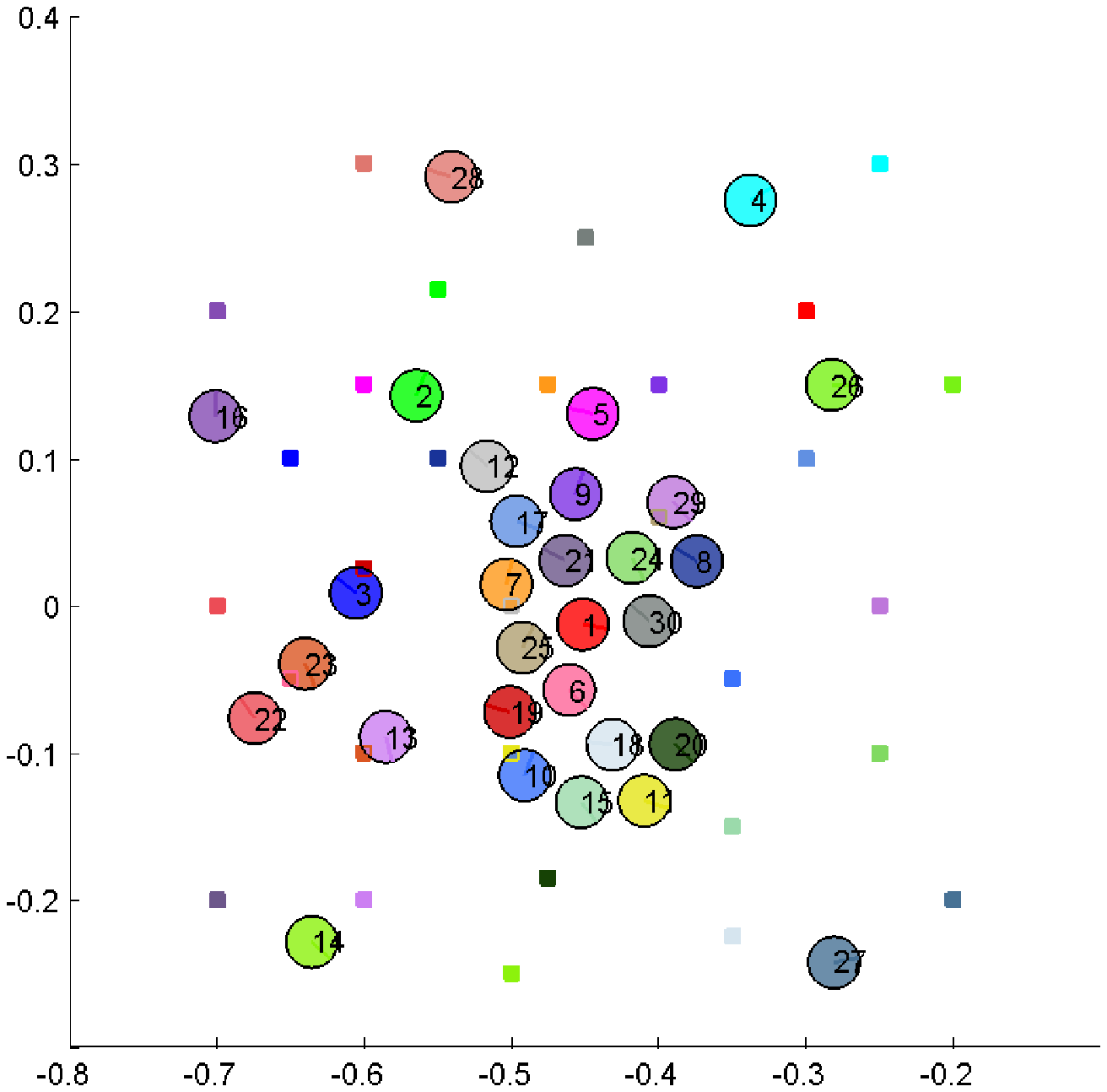}}
%\subfigure{\includegraphics[width=0.32\textwidth,clip]{Journal_SimT/jshot15.eps}}
\subfigure{\includegraphics[width=0.32\textwidth,clip]{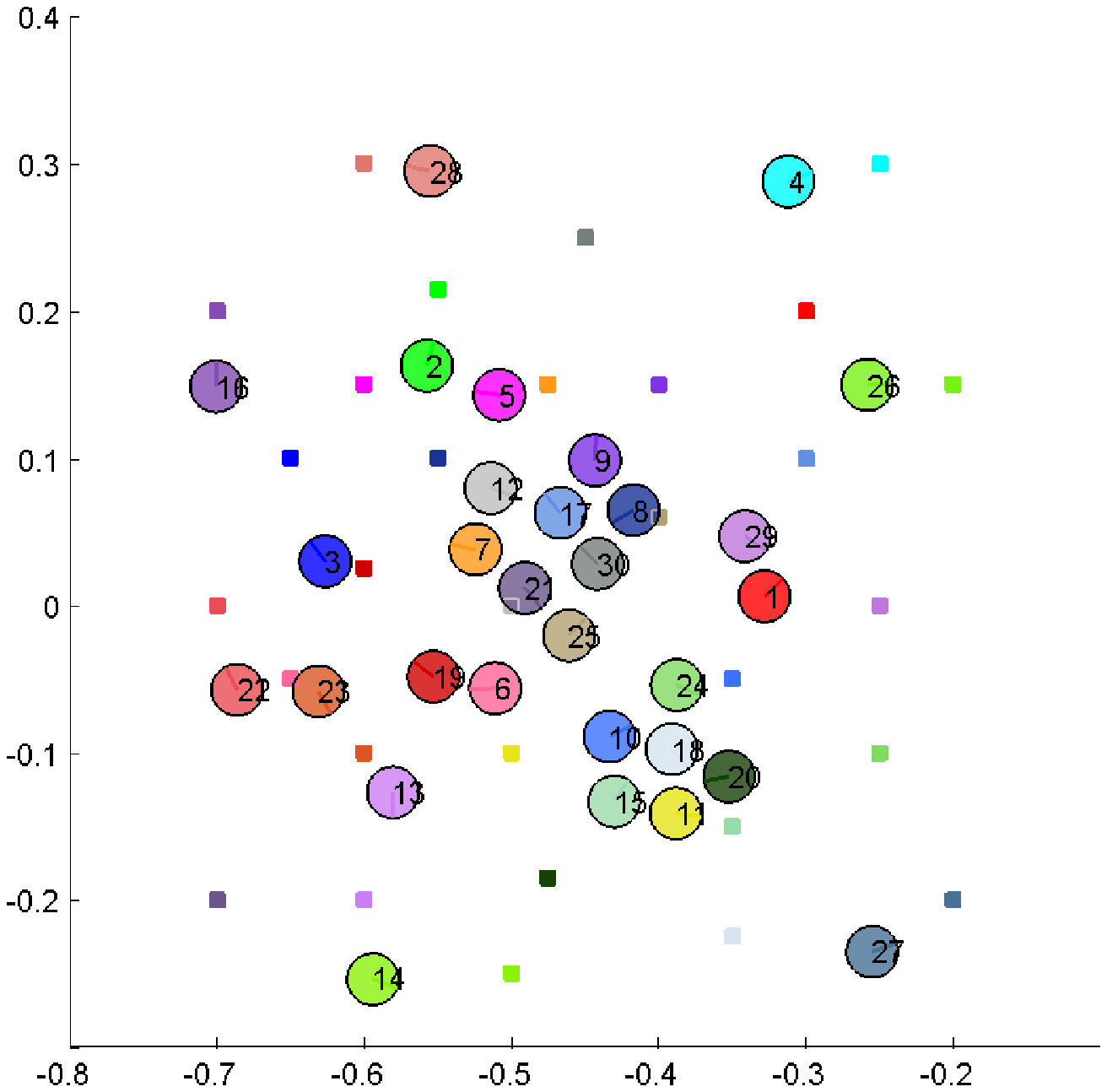}}
%\subfigure{\includegraphics[width=0.3\textwidth,clip]{Journal_SimT/jshot17.eps}}
%\subfigure{\includegraphics[width=0.3\textwidth,clip]{Journal_SimT/jshot18.eps}}
\subfigure{\includegraphics[width=0.32\textwidth,clip]{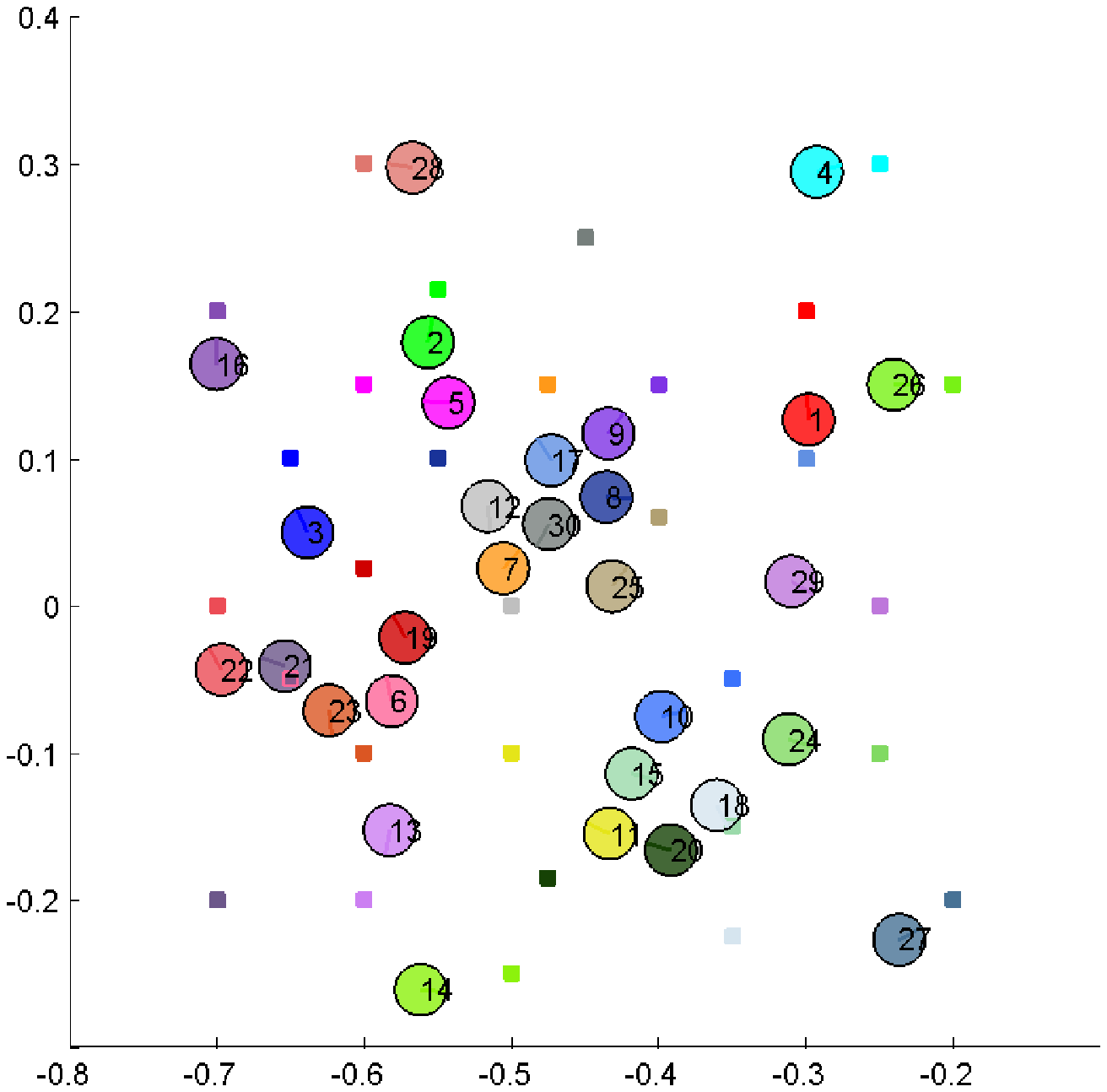}}
%\subfigure{\includegraphics[width=0.3\textwidth,clip]{Journal_SimT/jshot20.eps}}
%\subfigure{\includegraphics[width=0.3\textwidth,clip]{Journal_SimT/jshot21.eps}}
\subfigure{\includegraphics[width=0.32\textwidth,clip]{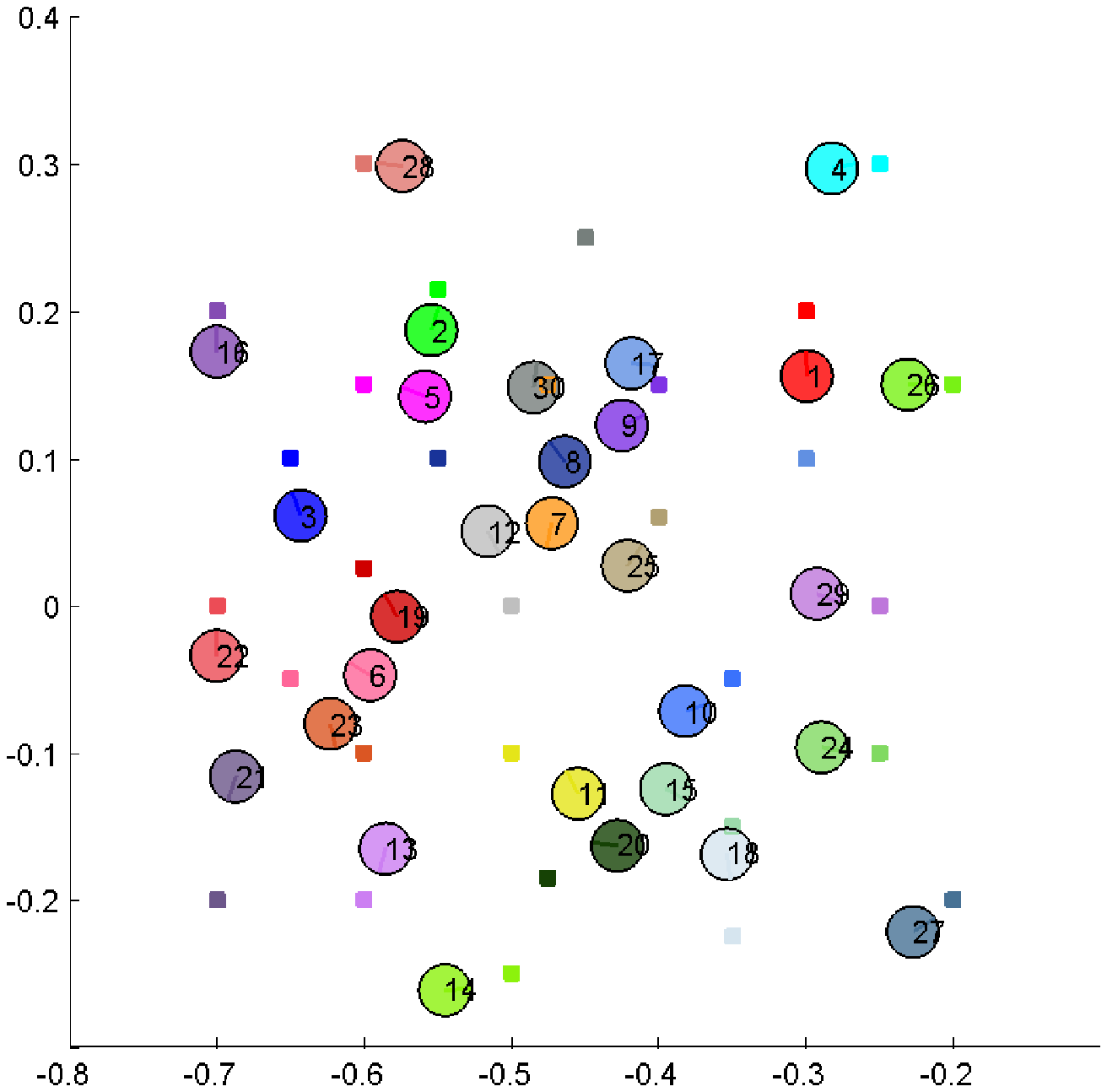}}
}
\caption{Collision-free motion of 30 nonholonomic agents under the proposed control strategy.}
\label{fig:multi_scenario1}
\end{figure*}

\begin{figure*}
{
\centering
%\subfigure{\includegraphics[width=0.3\textwidth,clip]{Journal_SimT/jshot23.eps}}
%\subfigure{\includegraphics[width=0.3\textwidth,clip]{Journal_SimT/jshot24.eps}}
\subfigure{\includegraphics[width=0.32\textwidth,clip]{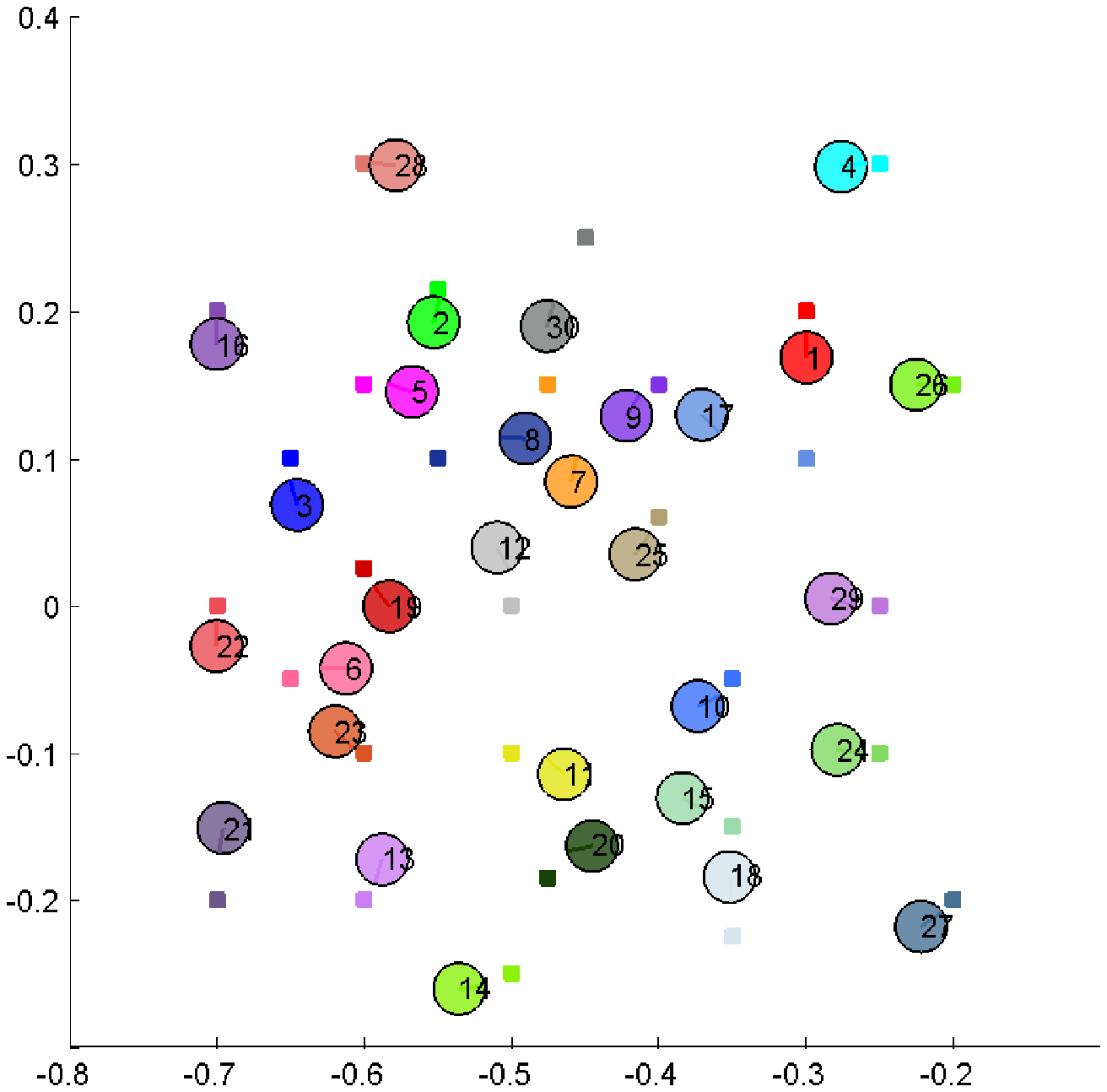}}
%\subfigure{\includegraphics[width=0.3\textwidth,clip]{Journal_SimT/jshot26.eps}}
%\subfigure{\includegraphics[width=0.3\textwidth,clip]{Journal_SimT/jshot27.eps}}
\subfigure{\includegraphics[width=0.32\textwidth,clip]{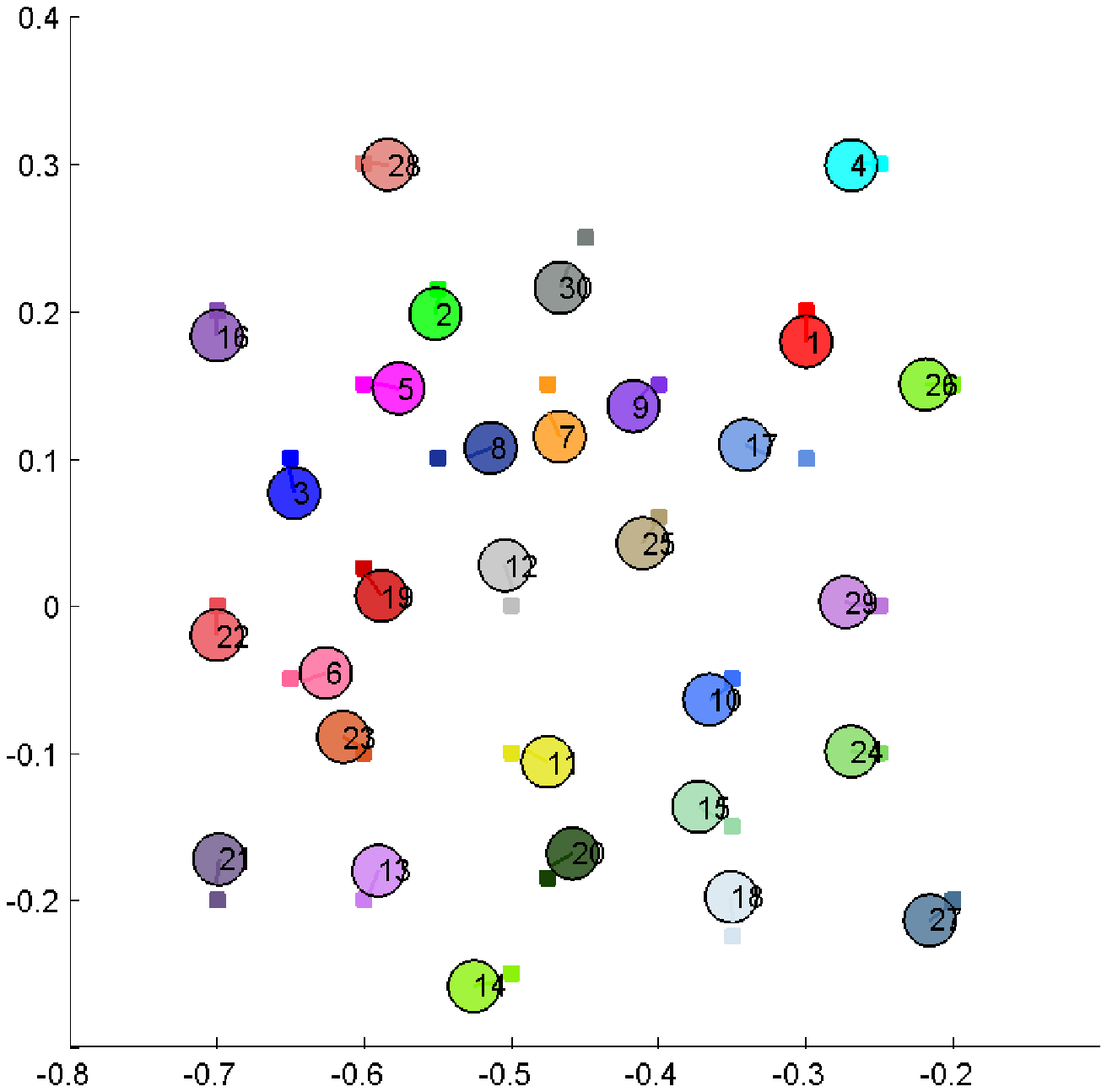}}
%\subfigure{\includegraphics[width=0.3\textwidth,clip]{Journal_SimT/jshot29.eps}}
%\subfigure{\includegraphics[width=0.3\textwidth,clip]{Journal_SimT/jshot30.eps}}
%\subfigure{\includegraphics[width=0.3\textwidth,clip]{Journal_SimT/jshot31.eps}}
%\subfigure{\includegraphics[width=0.3\textwidth,clip]{Journal_SimT/jshot32.eps}}
%\subfigure{\includegraphics[width=0.3\textwidth,clip]{Journal_SimT/jshot33.eps}}
%\subfigure{\includegraphics[width=0.3\textwidth,clip]{Journal_SimT/jshot34.eps}}
%\subfigure{\includegraphics[width=0.3\textwidth,clip]{Journal_SimT/jshot35.eps}}
\subfigure{\includegraphics[width=0.32\textwidth,clip]{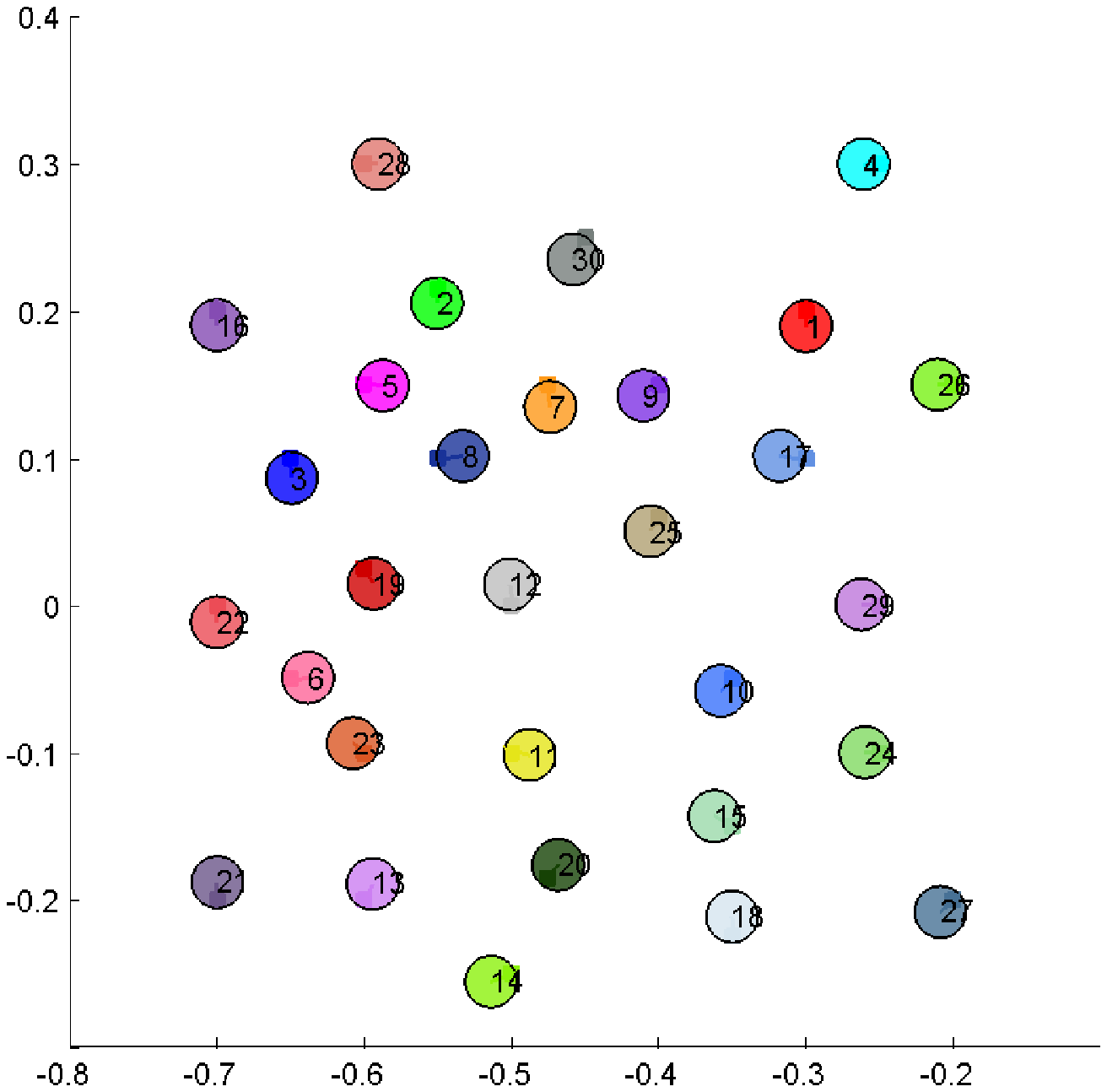}}
}
\caption{Collision-free motion of 30 nonholonomic agents under the proposed control strategy (Continued).}
\label{fig:multi_scenario2}
\end{figure*}

%%%%%%%%%%%%%%%%%%%%%%%%%%%%%%%%%%%%%%%%%%%
%\begin{figure*}
%{
%\centering
%\subfigure{\includegraphics[width=0.33\textwidth,clip]{shot1.eps}}
%\subfigure{\includegraphics[width=0.33\textwidth,clip]{shot2.eps}}
%\subfigure{\includegraphics[width=0.33\textwidth,clip]{shot3.eps}}
%\subfigure{\includegraphics[width=0.33\textwidth,clip]{shot4.eps}}
%\subfigure{\includegraphics[width=0.33\textwidth,clip]{shot5.eps}}
%}
%\caption{Collision-free motion of 4 nonholonomic agents. The considered communication links are undirected.}
%\label{fig:multiagent-und}
%\end{figure*}
%
%\begin{figure*}
%{
%\centering
%\subfigure{\includegraphics[width=0.33\textwidth,clip]{shot1_und.eps}}
%\subfigure{\includegraphics[width=0.33\textwidth,clip]{shot2_und.eps}}
%\subfigure{\includegraphics[width=0.33\textwidth,clip]{shot3_und.eps}}
%\subfigure{\includegraphics[width=0.33\textwidth,clip]{shot4_und.eps}}
%\subfigure{\includegraphics[width=0.33\textwidth,clip]{shot5_und.eps}}
%\subfigure{\includegraphics[width=0.33\textwidth,clip]{shot6_und.eps}}
%}
%\caption{Collision-free motion of 4 nonholonomic agents. The considered communication links are directed.}
%\label{fig:multiagent}
%\end{figure*}
%
%\begin{figure}
%\centering
%\includegraphics[width=0.9\columnwidth,clip]{dij.eps}
%\caption{Inter-agent distances remain greater than $2\varrho$.}
%\label{fig:distances-und}
%\end{figure}
%
%\begin{figure}
%\centering
%\includegraphics[width=0.9\columnwidth,clip]{dij_und.eps}
%\caption{Inter-agent distances remain greater than $2\varrho$.}
%\label{fig:distances}
%\end{figure}
%%%%%%%%%%%%%%%%%%%%%%%%%%%%%%%%%%%%%%%%%%

\section{Conclusions}\label{Conclusions}

This paper presented a novel methodology for the motion planning of unicycle robots in environments with obstacles, with extensions to the collision avoidance in multi-agent systems. The method is based on a family of vector fields whose integral curves exhibit attractive or repulsive behavior depending on the value of a parameter. It was shown that attractive-to-the-goal and repulsive-around-obstacles vector fields can be suitably blended in order to yield almost global feedback motion plans in environments with circular obstacles. The case of collision avoidance under local sensing/communication in multi-agent scenarios was also treated. No parameter tuning is needed in order to avoid local minima, as needed in similar methods which are based on scalar (potential) functions. Current work focuses on the definition of vector fields encoding input constraints, such as curvature bounds, which may be more appropriate for aircraft and car-like vehicles. %Furthermore, the pattern of the integral curves for various values of the parameter $\lambda$ is under consideration, in order to possibly encode more complicated obstacle environments.

\bibliographystyle{ieeetran}
\bibliography{myreferences}

\appendix
Here we present some preliminary ideas on the extension of the method to polygonal environments. 

Consider the pattern of the integral curves for $\lambda<-1$, shown in Fig. \ref{fig:polygonal_flow}. The repulsive nature of the integral curves \ac{wrt} the axis the vector $\bm p$ lies on can be used to define a repulsive flow \ac{wrt} each side of polygonal obstacles, as shown in Fig. \ref{fig:polygonal_obstacle}. The effect of the repulsive flows can be confined around the polygonal obstacle using blending mechanisms as those presented in Section 4. Identifying sufficient minimum clearance around the obstacles which guarantees the almost global convergence of the integral curves to a goal configuration in such a polygonal environment is currently ongoing work, and beyond the scope and the length of the current paper.
\begin{figure}
\centering
\includegraphics[width=0.5\columnwidth,clip]{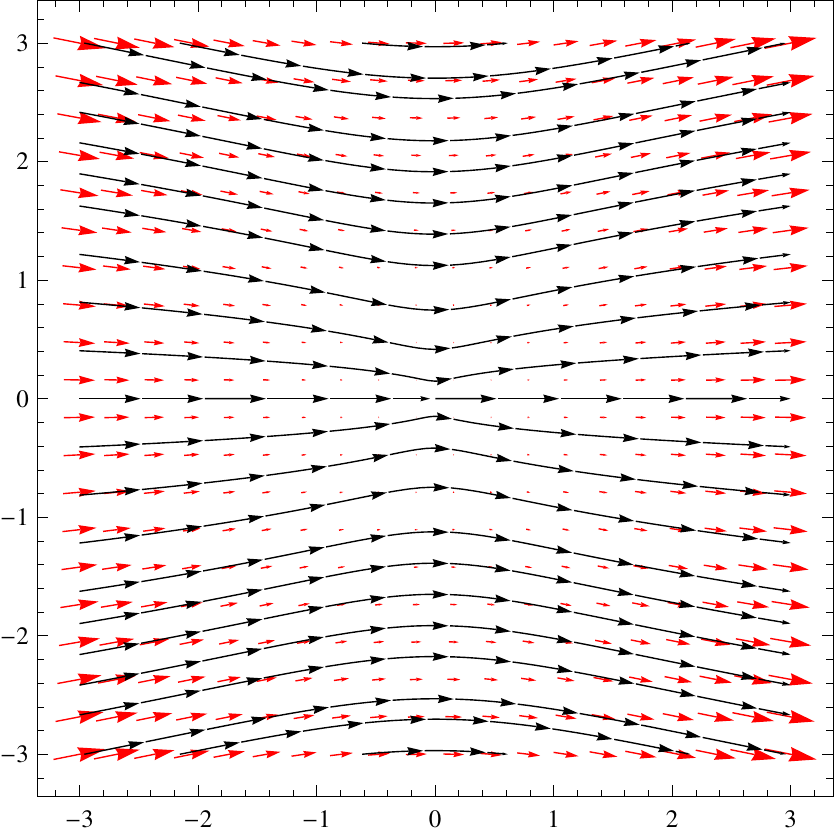}
\caption{The integral curves for $\lambda=-1$, $p_x=-1$, $p_y=0$.}
\label{fig:polygonal_flow}
\end{figure}

\begin{figure}
\centering
\includegraphics[width=0.5\columnwidth,clip]{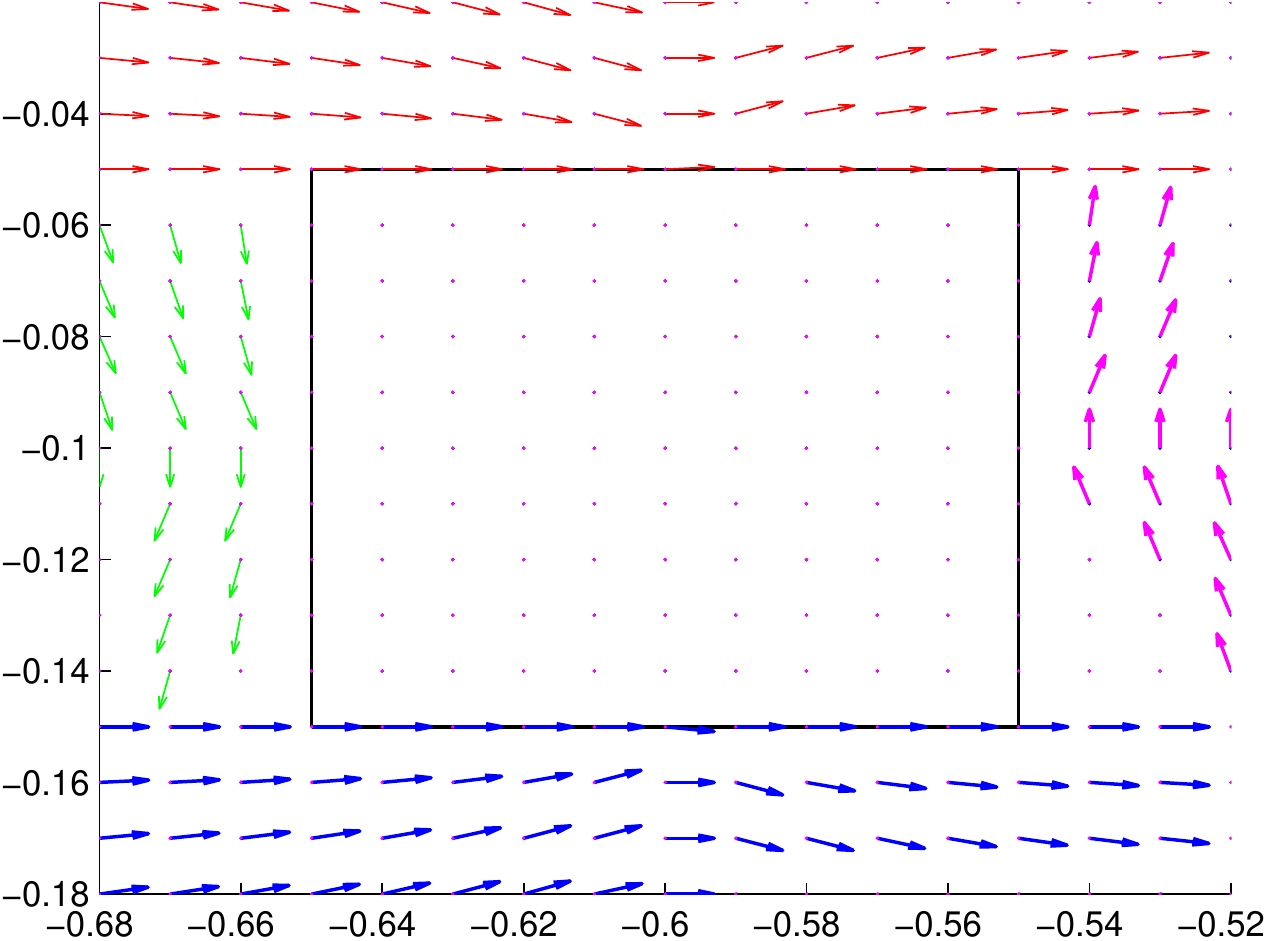}
\caption{Integral curves around a polygonal obstacle.}
\label{fig:polygonal_obstacle}
\end{figure}

%\section{Appendix 2}
%Let us consider the case $\nabla (\nabla T_i^+) \; \mathbf F=0 \Rightarrow$
%\begin{align*}
%2\sigma_i &+\frac{4\sigma_i}{\|\mathbf F_g\|}\left(\begin{bmatrix}y_o&-x_o\end{bmatrix}\mathbf F_g\right)\left(\begin{bmatrix}y&-x\end{bmatrix}\mathbf F\right)=0\\
%2\sigma_i &+\frac{4\sigma_i}{\|\mathbf F_g\|}\|\bm r_o\| \|\mathbf F_g\| \cos\alpha \; \|\bm r\| \|\mathbf F\| \cos\beta = 0,
%\end{align*}
%where $\alpha$ is the angle between the vectors $\left[y_o\;\;-x_o\right]^T$, $\mathbf F_g$, and $\beta$ is the angle between the vectors $\left[y\;\;-x\right]^T$, $\mathbf F$, respectively.
%
%If the quantity $E=\cos\alpha \cos\beta\geq 0$, then $\nabla (\nabla T_i^+) \; \mathbf F=0$ does not have a real solution (when $E>0$), or has a real solution for $\sigma_i=0$ (when $E=0$), a contradiction since $\sigma_i\in(0,1)$. Let us assume that $E=\cos\alpha \cos\beta<0$, which reads:
%\begin{align*}
%2\sigma_i + 4\sigma_i \; \|\bm r_o\| \; \|\bm r\| \; \|\mathbf F\| \; E = 0,
%\end{align*}
%with $E\in[-1,0)$ and $\mathbf F\in(0,1]$ by construction. Denote $Z=\|\mathbf F\| \; E$, where $Z\in[-1,0)$, to further write:
%\begin{align*}
%2\sigma_i (1 + 2 \; \|\bm r_o\| \; \|\bm r\| \; Z) = 0 \Rightarrow Z = -\frac{1}{ 2 \; \|\bm r_o\| \; \|\bm r\|}.
%\end{align*}
%Since $Z\in[-1,0)$, the condition above further implies:
%\begin{align*}
%-\frac{1}{ 2 \; \|\bm r_o\| \; \|\bm r\|} >-1 \quad \Rightarrow \quad \|\bm r_o\| \; \|\bm r\| > \frac{1}{2}
%\end{align*}

\end{document}